\documentclass[dvipsnames,table,sigconf]{acmart}

\AtBeginDocument{%
  }

\copyrightyear{2025}
\acmYear{2025}
\setcopyright{cc}
\setcctype{by}
\acmConference[MM '25]{Proceedings of the 33rd ACM International Conference
on Multimedia}{October 27--31, 2025}{Dublin, Ireland}
\acmBooktitle{Proceedings of the 33rd ACM International Conference on
Multimedia (MM '25), October 27--31, 2025, Dublin,
Ireland}\acmDOI{10.1145/3746027.3755142}
\acmISBN{979-8-4007-2035-2/2025/10}

\usepackage[capitalise]{cleveref}
\usepackage{tcolorbox}
\usepackage{amsmath} 
\DeclareMathOperator*{\argmax}{arg\,max}
\usepackage{bm}
\usepackage{multirow}

\begin{document}

\newcommand{\methodName}{MiraGe}
\title{
\methodName: Multimodal Discriminative Representation Learning for Generalizable AI-Generated Image Detection
}

\author{Kuo Shi}
\affiliation{%
  \institution{University of Technology Sydney}
  \city{Ultimo}
  \state{NSW}
  \country{Australia}
}
\email{kuo.shi@student.uts.edu.au}

\author{Jie Lu}
\affiliation{%
  \institution{University of Technology Sydney}
  \city{Ultimo}
  \state{NSW}
  \country{Australia}
}
\email{jie.lu@uts.edu.au}

\author{Shanshan Ye}
\affiliation{%
  \institution{University of Technology Sydney}
  \city{Ultimo}
  \state{NSW}
  \country{Australia}
}
\email{shanshan.ye@student.uts.edu.au}

\author{Guangquan Zhang}
\affiliation{%
  \institution{University of Technology Sydney}
  \city{Ultimo}
  \state{NSW}
  \country{Australia}
}
\email{guangquan.zhang@uts.edu.au}

\author{Zhen Fang}
\authornote{Corresponding author.}
\affiliation{%
  \institution{University of Technology Sydney}
  \city{Ultimo}
  \state{NSW}
  \country{Australia}
}
\email{zhen.fang@uts.edu.au}

\begin{abstract}
Recent advances in generative models have highlighted the need for robust detectors capable of distinguishing real images from AI-generated images. 
While existing methods perform well on known generators, their performance often declines when tested with newly emerging or unseen generative models due to overlapping feature embeddings that hinder accurate cross-generator classification.
In this paper, we propose \emph{\textbf{M}ult\textbf{i}modal Disc\textbf{r}iminative Represent\textbf{a}tion Learning for \textbf{Ge}neralizable AI-generated Image Detection} (\methodName), a method designed to learn generator-invariant features. 
Motivated by theoretical insights on intra-class variation minimization and inter-class separation, \methodName\ tightly aligns features within the same class while maximizing separation between classes, enhancing feature discriminability. 
Moreover, we apply multimodal prompt learning to further refine these principles into CLIP, leveraging text embeddings as semantic anchors for effective discriminative representation learning, thereby improving generalizability.
Comprehensive experiments across multiple benchmarks show that \methodName\ achieves state-of-the-art performance, maintaining robustness even against unseen generators like Sora.
\end{abstract}




\maketitle

\section{Introduction}

\begin{figure}[t]
\begin{center}
\centerline{
\includegraphics[page=1, width=\columnwidth]{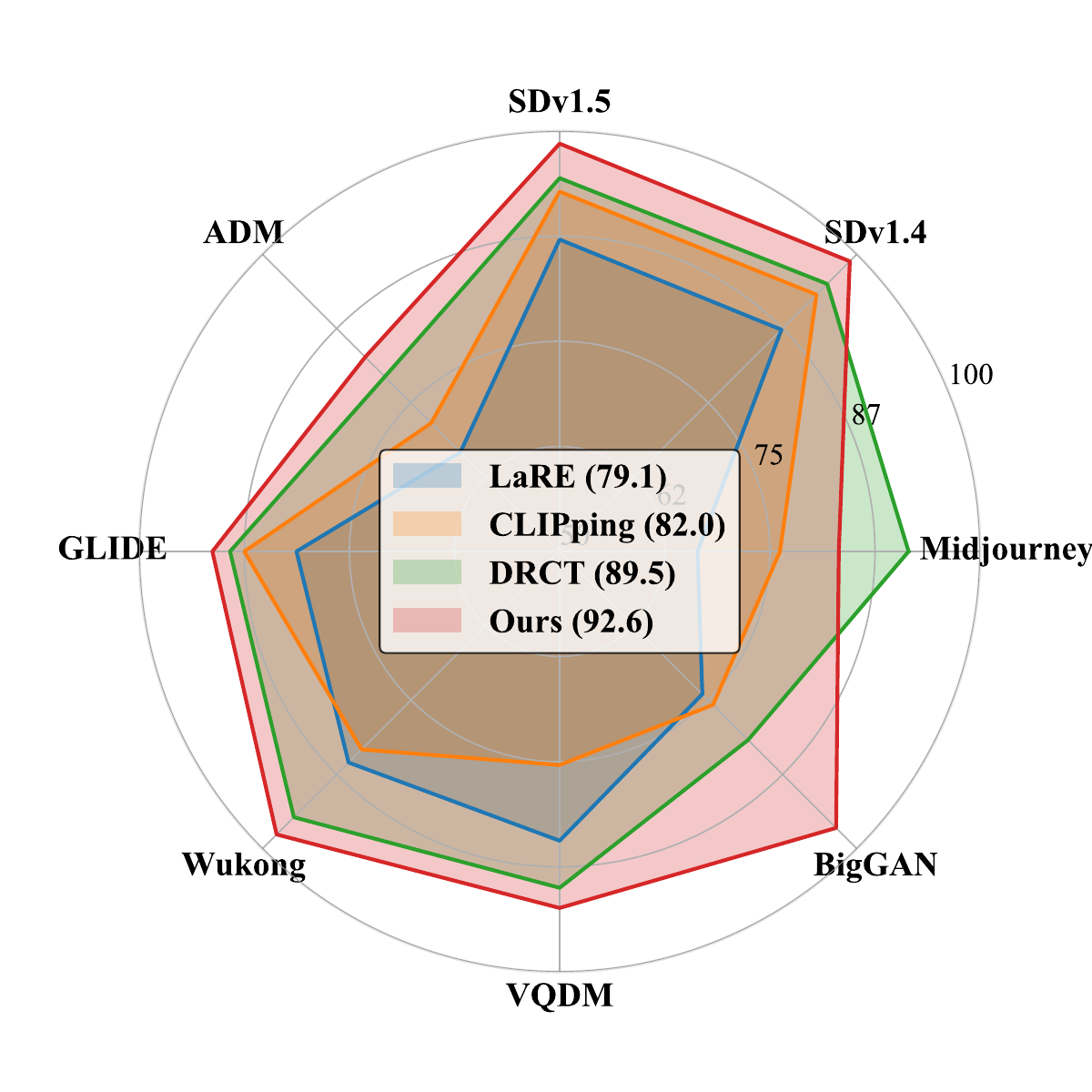}
}
\Description{Comparison of generalization performance between the proposed method and existing detection methods. }
\vspace{-0.2in}
\caption{
Comparison of generalization performance between the proposed method and existing detection methods. 
All detection methods were trained on a dataset consisting of generated images from Stable Diffusion (SD) v1.4 and real images from the MSCOCO dataset. The reported detection accuracies were evaluated on eight subsets of the GenImage dataset. Results demonstrate that the proposed method outperforms all other methods.
}
\label{fig:radar}
\end{center}
\vspace{-0.15in}
\end{figure}

\begin{figure*}[t]
\raggedright{
\begin{minipage}[t]{0.3\linewidth}
    \raggedright
    \includegraphics[page=1, height=4.3cm, keepaspectratio]{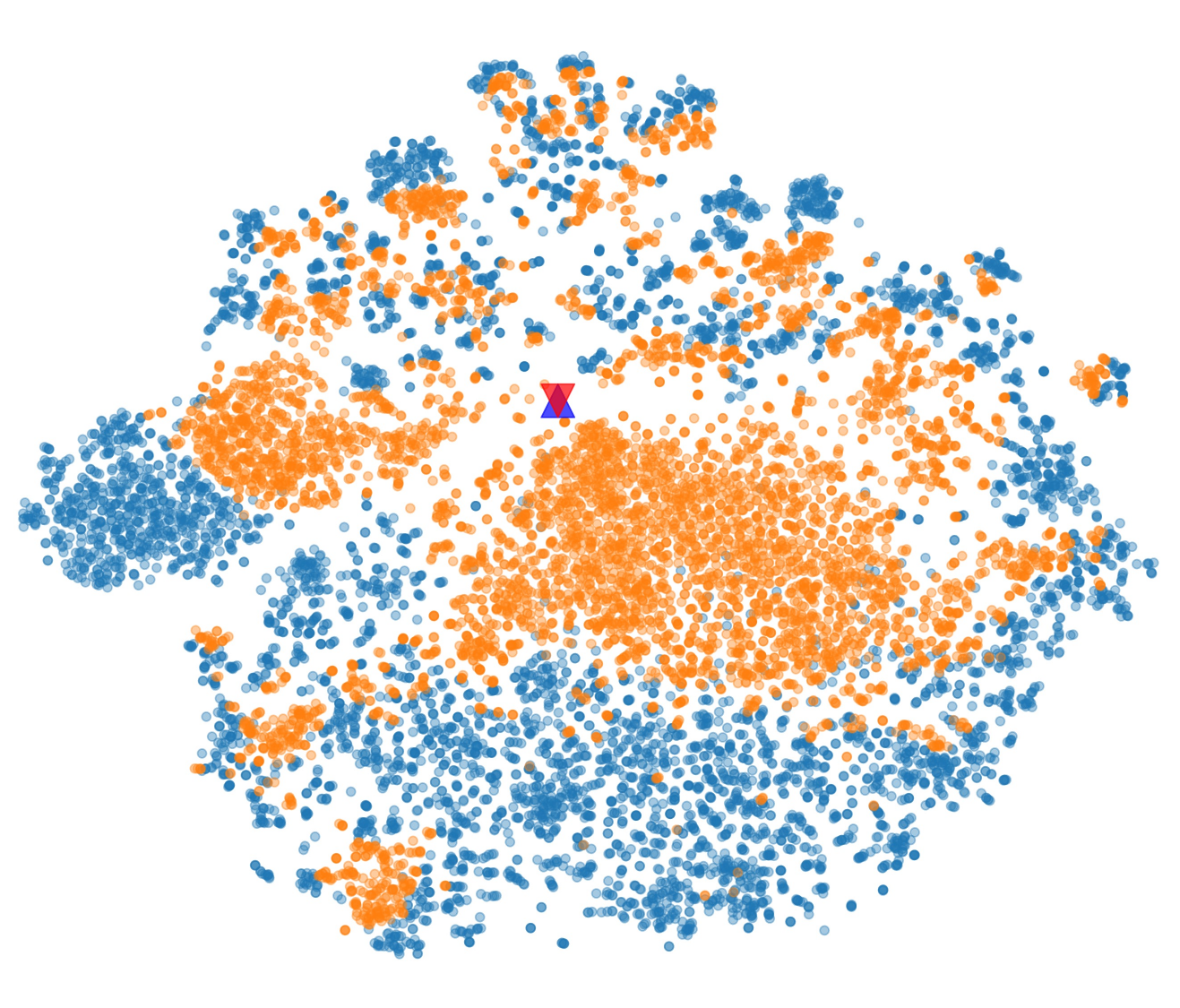}
    \label{fig:tsne_zsclip}
\end{minipage}
%\hspace{0.1cm}
\begin{minipage}[t]{0.3\linewidth} 
    \raggedright
    \includegraphics[page=2, height=4.3cm, keepaspectratio]{figures/ICML2025_TSNE.pdf}
    \label{fig:tsne_clipping}
\end{minipage}
\begin{minipage}[t]{0.3\linewidth} 
    \raggedright
    \includegraphics[page=3, height=4.3cm, keepaspectratio]{figures/ICML2025_TSNE.pdf}
    \label{fig:tsne_ours}
\end{minipage}
}
\vspace{-0.1in}
\Description{Visualization of t-SNE embeddings.}
\caption{
{Visualization of t-SNE embeddings.} 
(a) Zero-shot CLIP \cite{DBLP:conf/icml/RadfordKHRGASAM21}, (b) CLIPping \cite{DBLP:conf/mir/KhanD24}, and (c) \methodName~(Ours). 
While CLIPping modifies the zero-shot CLIP text features, \methodName\ instead treats text features as semantic centers, pulling same-class samples closer and pushing different-class samples apart. 
These principles yield the highest detection accuracy. 
All models are trained on images generated by Stable Diffusion 1.4 and tested on BigGAN images from the GenImage dataset.
}
\label{fig:tsne_compare}
\end{figure*}  

Recent advancements in generative models (e.g., Stable Diffusion \cite{DBLP:conf/cvpr/RombachBLEO22}, DALL-E 3 \cite{betker2023improving} and Sora \cite{brooks2024video}) have revolutionized visual content creation, enabling a wide range of applications in digital art, advertising, and entertainment. However, these powerful models also pose risks of misuse \cite{DBLP:conf/nips/LuH0QWLO23}, such as fabricating fake news, manipulating public opinion, and infringing on copyrights. Consequently, developing robust detection methods to distinguish real images from AI-generated ones has become a critical requirement for maintaining a trustworthy cyberspace environment.

A common approach to distinguish real and fake images is training a binary classifier, which performs well on seen generators but often fails on unseen ones. To improve generalization, methods like CNNDet \cite{DBLP:conf/cvpr/WangW0OE20} enhance robustness through data augmentation, while UnivFD \cite{DBLP:conf/cvpr/OjhaLL23} and CLIPping \cite{DBLP:conf/mir/KhanD24} leverage CLIP's feature space with techniques such as prompt learning and linear probing. However, most methods fail to explicitly separate the feature distributions of real and fake images, resulting in overlapping embeddings that hinder accurate classification for unseen generators.

In this work, we address the challenge of generalizing AI generated image detection by leveraging discriminative representation learning \cite{DBLP:journals/tmm/GouYYYY23}. 
Guided by the principles of minimizing intra-class variation and maximizing inter-class separation, our method clusters features within the same class while separating features across different classes. 
This learning objective promotes high cosine similarity within same-class embeddings and lower similarity across different classes, effectively separating feature distributions and enhancing feature discriminability.

Building on these principles, we propose a novel method {M}ult{i}modal Disc{r}iminative Represent{a}tion Learning for {Ge}neralizable AI-generated Image Detection (\methodName), which applies multimodal prompt learning to align visual and semantic representations for robust, generator-agnostic detection. 
By using text embeddings as stable semantic anchors (e.g., ``Real" or ``Fake"), \methodName\ refines discriminative representation learning and improves generalizability. 
Unlike image-only methods, our multimodal design grounds visual features in text-driven semantics, enabling superior generalization.

To demonstrate the effectiveness of our method, \cref{fig:radar} compares the detection accuracies of our proposed \methodName\ against existing methods, including DRCT \cite{DBLP:conf/icml/ChenZYY24}, CLIPping \cite{DBLP:conf/mir/KhanD24}, and LaRE \cite{DBLP:conf/cvpr/LuoDYD24}, across all subsets of the GenImage dataset \cite{DBLP:conf/nips/ZhuCYHLLT0H023}.
\methodName\ achieves a 92.6\% accuracy, surpassing all baselines and highlighting the effectiveness of our multimodal method for improved generalizability.

To further illustrate \methodName's effectiveness in discriminative representation learning, \cref{fig:tsne_compare} compares t-SNE embeddings from CLIP, CLIPping, and \methodName. 
When tested on unseen generators, \methodName\ consistently aligns images with their corresponding text embeddings, maintains distinct class boundaries, and exhibits robust adaptability to domain shifts and unseen generative models.

In summary, our main contributions are:
\begin{itemize}
    \item We introduce \methodName, which applies multimodal feature alignment to foster discriminative representation learning, effectively minimizing intra-class variation and maximizing inter-class separation, thereby enhancing generalizability to unseen generative models.
    \item We perform comprehensive experiments across multiple benchmarks and real-world scenarios, demonstrating that \methodName\ achieves {robust}, {accurate}, and {transferable} performance in AI-generated image detection.
    \item We further {validate \methodName\ on state-of-the-art generators}, including {Sora} \cite{brooks2024video}, DALL-E 3 \cite{betker2023improving} and Infinity \cite{han2024infinityscalingbitwiseautoregressive}, showcasing its effectiveness in handling {previously unseen models} and emphasizing its generalizability.
\end{itemize}

\section{Related Work}
\subsection{AI-generated Images Detection}
In recent years, the rapid advancement of generative models has intensified research on AI-generated image detection, as these models can produce strikingly realistic images that raise concerns over misinformation, privacy, and authenticity. 
Early work often relied on specialized binary classifiers; for instance, CNNDet \cite{DBLP:conf/cvpr/WangW0OE20} directly classifies images as real or fake using a convolutional neural network. 
Several methods focus on frequency-domain analysis to detect inconsistencies \cite{DBLP:conf/wifs/0022KC19, DBLP:conf/eccv/LiuYBXLG22}. 
Others emphasize local artifacts rather than global semantics; 
Patchfor \cite{DBLP:conf/eccv/ChaiBLI20} uses classifiers with limited receptive fields to capture local defects, whereas Fusing \cite{DBLP:conf/icip/JuJKXNL22} adopts a dual-branch design combining global spatial information with carefully selected local patches. 
NPR \cite{DBLP:conf/cvpr/TanLZWGLW24} leverages spatial relations among neighboring pixels, and LGrad \cite{DBLP:conf/cvpr/Tan0WGW23} generates gradient maps using a pre-trained CNN, both strategies targeting low-level artifacts. 
AIDE \cite{DBLP:journals/corr/abs-2406-19435} further integrates multiple experts to extract visual artifacts and noise patterns, selecting the highest and lowest-frequency patches to detect based on low-level inconsistencies.

Another line of research focuses on reconstruction-based detection \cite{DBLP:conf/iccv/WangBZWHCL23, DBLP:conf/cvpr/LuoDYD24}. 
For example, DRCT \cite{DBLP:conf/icml/ChenZYY24} generates hard samples by reconstructing real images through a diffusion model and then applies contrastive learning to capture artifacts.

Recent works have leveraged CLIP-derived features for improved detection, as exemplified by UnivFD \cite{DBLP:conf/cvpr/OjhaLL23}, which trains a classifier in CLIP’s representation space, FAMSeC \cite{DBLP:journals/spl/XuYFLZ25} applies an instance-level, vision-only contrastive objective, and CLIPping \cite{DBLP:conf/mir/KhanD24}, which applies prompt learning and linear probing on CLIP’s encoders. 
While these methods show promise, they still struggle to generalize to unseen models, and focusing on a single modality in CLIP can be suboptimal. 
To address these issues, we propose a method that simultaneously optimizes image and text features using discriminative representation learning, thereby capturing generator-agnostic characteristics and enhancing generalization.

\subsection{Pre-trained Vision-Language Models} 
Recently, large-scale pre-trained models that integrate both image and language modalities have achieved remarkable success, demonstrating robust performance across a variety of tasks \cite{DBLP:journals/pami/ZhangHJL24}. 
These models attract attention for their strong zero-shot capabilities and robustness to distribution shifts. 
Among them, Contrastive Language-Image Pretraining (CLIP) \cite{DBLP:conf/icml/RadfordKHRGASAM21} stands out as a large-scale approach exhibiting exceptional zero-shot ability on tasks such as image classification \cite{DBLP:journals/tfs/ShiLFZ24, DBLP:conf/mm/WangZ0024, DBLP:conf/ijcnn/WangZFL24} and image-text retrieval \cite{DBLP:conf/mm/MaXSYZJ22}. 

Although CLIP demonstrates impressive zero-shot performance, further fine-tuning is often required to reach state-of-the-art accuracy on specific downstream tasks. For instance, on the simple MNIST dataset \cite{DBLP:journals/spm/Deng12}, the zero-shot CLIP model (ViT-B/16) achieved only 55\% accuracy. 
However, fully fine-tuning CLIP on a downstream dataset compromises its robustness to distribution shifts \cite{DBLP:conf/cvpr/WortsmanIKLKRLH22}. 
To address this issue, numerous studies have proposed specialized fine-tuning strategies for CLIP. One example is CoOp \cite{DBLP:journals/ijcv/ZhouYLL22}, which injects learnable vectors into the textual prompt context and optimizes these vectors during fine-tuning while freezing CLIP’s vision and text encoders. 
Nevertheless, focusing solely on the text branch may lead to suboptimal performance. Consequently, MaPLe \cite{DBLP:conf/cvpr/KhattakR0KK23} extends prompt learning to both the vision and language branches, thereby enhancing alignment between these representations. 
Building on MaPLe’s approach, we incorporate our discriminative representation learning on multimodal to address generalization challenges in AI-generated image detection.
A more comprehensive discussion of related work appears in \cref{sec:appendix_relatedwork}.

\section{Preliminaries}\label{Sec::Pre}

\textbf{AI-generated image detection.} Let $\mathcal{S}_{{tr}}^N = \{\mathbf{x}_i^N\}_{i=1}^{n}$ and $\mathcal{S}_{{tr}}^G = \{\mathbf{x}_j^G\}_{j=1}^{m}$ be the training images collected from natural environments and generated by the generative AI models, respectively. 
The combined training set is denoted as $\mathcal{S}_{{tr}}$. 
Following \citet{DBLP:conf/cvpr/WangW0OE20}, the AI-generated image detection task can be defined as follows:
\begin{tcolorbox}[colframe=black!75, colback=gray!10, sharp corners]
\textbf{Problem 1 (AI-generated Image Detection.)} \\
AI-generated image detection aims to learn a detector $\mathbf{f}$ using available resources (e.g., training images $\mathcal{S}_{{tr}}^N$, $\mathcal{S}_{{tr}}^G$ and pre-trained models) such that $\mathbf{f}$ can answer whether a given image $\mathbf{x}$ is natural or AI-generated accurately.
\end{tcolorbox}

Current benchmarks \cite{DBLP:conf/cvpr/WangW0OE20, DBLP:conf/cvpr/OjhaLL23, DBLP:conf/nips/ZhuCYHLLT0H023} typically involve training and validating the detector on images generated by a single known model, followed by testing on images from multiple unseen generative models. This setup emphasizes the generalization challenge, as detectors must adapt to unknown models. 
We leverage the pre-trained CLIP model, which offers rich textual information and a multimodal foundation, as a basis for addressing this challenge.

\textbf{CLIP model} \cite{DBLP:conf/icml/RadfordKHRGASAM21}\textbf{.} 
Given any image $\mathbf{x}$ and label $y\in \mathcal{Y}$, where $\mathcal{Y}=\{\text{Real}, \text{Fake}\}$ in AI-generated image detection, we use CLIP to extract features of $\mathbf{x}$ and $y$ through its image encoder $\mathbf{f}^{\rm img}$ and text encoder $\mathbf{f}^{\rm text}$. Following \cite{DBLP:conf/iclr/Jiang000LZ024}, the extracted image feature $\mathbf{h}\in \mathbb{R}^d$ and text feature $\mathbf{e}_y\in \mathbb{R}^d$ are given by: 
\begin{equation}
\label{eq:pre_feature_extract}
    \mathbf{h} = \mathbf{f}^{\rm img}(\mathbf{x}),~~
    \mathbf{e}_y = \mathbf{f}^{\rm text}({\rm Prompt}(y)),
\end{equation}
where ${\rm Prompt}(y)$ represents the prompt template for the labels, such as ``a photo of a Real'' or ``a photo of a Fake.''

In zero-shot classification, the goal is to predict the correct label for an image without prior task-specific training. 
CLIP performs this prediction by computing the cosine similarity $\langle\cdot,\cdot\rangle$ between the image embedding $\mathbf{h}$ and the text embeddings $\mathbf{e}$. The predicted label $\hat{y}$ is then obtained by selecting the label with the highest similarity:
\begin{equation}
    \hat{y} = \argmax_{y \in \mathcal{Y}} \langle\mathbf{h}, \mathbf{e}_y\rangle.
\end{equation}

\begin{figure*}[!t]
\centering{
\includegraphics[page=1, width=0.85\textwidth]{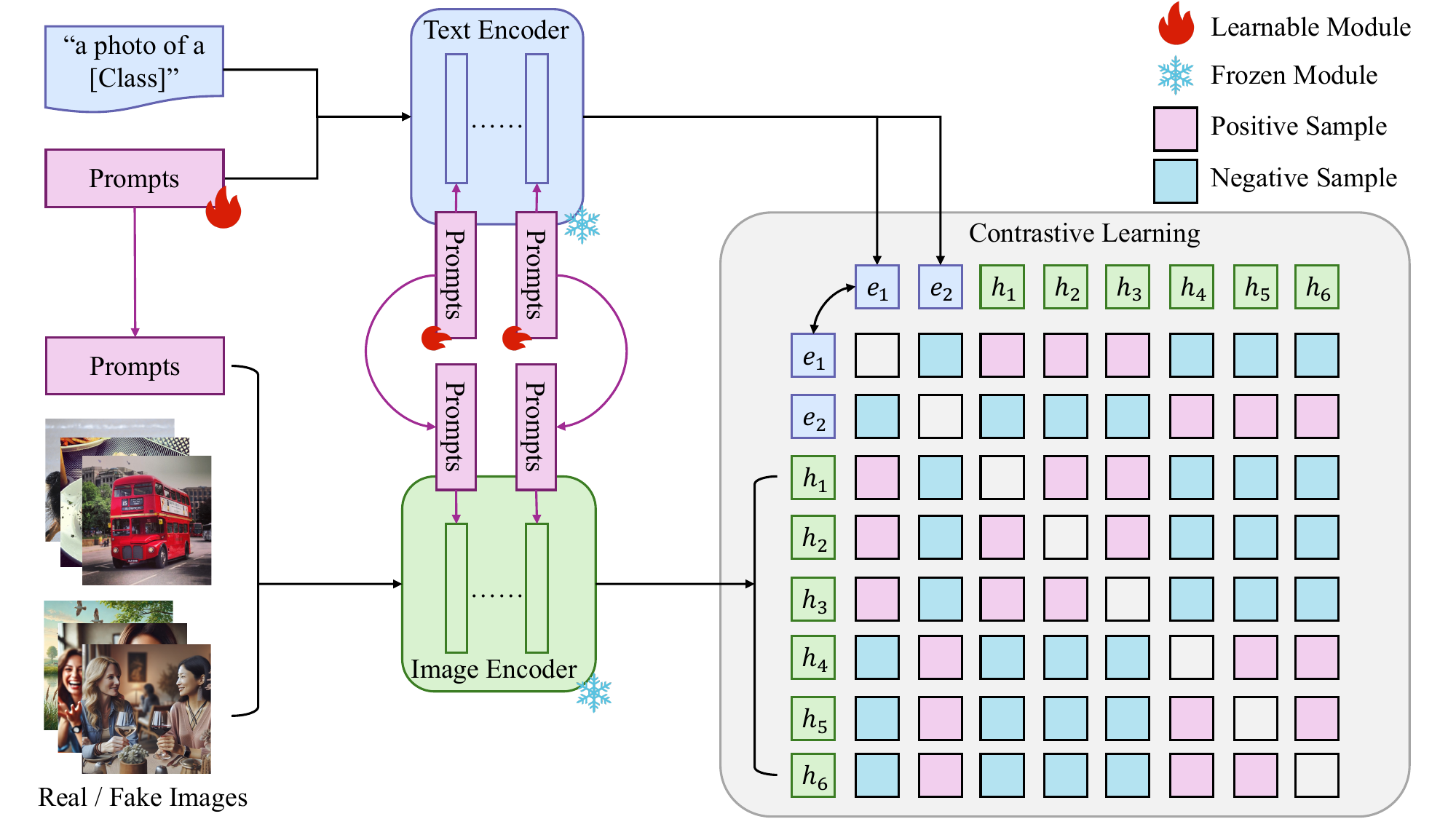}
}
\Description{Overview of our proposed method MiraGe.}
\caption{
{Overview of our proposed method MiraGe.} 
We illustrate two text embeddings, \(\mathbf{e}_1=\mathbf{e}_{\mathrm{Real}}\) and \(\mathbf{e}_2=\mathbf{e}_{\mathrm{Fake}}\), serving as text anchors for real and fake classes, respectively. 
We give example images from the “Real” class are mapped to \(\{\mathbf{h}_1, \mathbf{h}_2, \mathbf{h}_3\}\), and example images from the “Fake” class are mapped to \(\{\mathbf{h}_4, \mathbf{h}_5, \mathbf{h}_6\}\). 
Our multimodal prompt learning injects learnable prompts into both the text and image encoders while keeping the encoders themselves {frozen}.
We then apply discriminative representation learning over all embeddings:
{positive pairs} arise if two embeddings share the same label, e.g.\ \((\mathbf{e}_{\mathrm{Real}},\mathbf{h}_i)\) if \(\mathbf{h}_i\) is real, \((\mathbf{h}_1,\mathbf{h}_2)\) if both are real images, etc., and {negative pairs} if their labels differ. 
By {pulling} positive pairs closer and {pushing} negative pairs apart, \methodName\ achieves greater feature discriminability and robustly adapts to newly emerging generative models.
}
\label{fig:overview}
\end{figure*}

\section{Motivation from Theoretical Observations}
\label{sec:motivation}
Our work is inspired by the theoretical findings in \citet{DBLP:conf/nips/YeXCLLW21}, which identify intra-class variation minimization and inter-class separation as key properties for achieving superior generalization. 
While previous studies have primarily focused on single-modal settings, we leverage the large-scale pretraining of CLIP, which encodes rich cross-domain knowledge and provides robust text and image embeddings. 
In this work, we further refine these theoretical insights to the multimodal setting, exploring their application to the generalization of AI-generated image detection.

Specifically, we designate ``{$\text{Real}$}'' and ``{ $\text{Fake}$}'' as two separate classes and obtain their textual embeddings using the prompt template ``{a photo of a}''.
By aligning each image representation with the corresponding text anchor through multimodal prompt learning, we reduce intra-class variation (i.e., pulling same-class features closer) and {increase inter-class separation} (i.e., pushing different-class features apart). 
Below, we provide our theoretical basis.

\textbf{Notations.} Let $\mathcal{X}$ denote the image space, with $P_X$ representing the distribution defined over $\mathcal{X}$. We use the distribution $P_X$ to model the AI-generated image models and use $Q_X$ to model the natural distribution, which samples natural images. We define $\mathscr{D}_G$ as the set of all available AI-generated image models during testing. Additionally, let $\mathscr{D}_N$ represent the set of natural distributions.

\begin{definition}[CLIP-based Intra-class Variation]\label{def:clip_intra1}
{\itshape
The variation of the CLIP model across distributions $P_X, Q_X$ is 
\begin{equation}
\label{eq:clip_variation}
\begin{split}
&\mathcal{V}_{\rm CLIP}( \mathbf{f}^{\rm img},\mathbf{f}^{\rm text};P_X,Q_X)\\=&\max \big \{\mathcal{V}(  \mathbf{f}^{\rm img},\mathbf{f}^{\rm text};P_X),\mathcal{V}( \mathbf{f}^{\rm img},\mathbf{f}^{\rm text};Q_X) \big \},
\end{split}
\end{equation}
where
\begin{equation*}
\begin{split}
    \mathcal{V}(  \mathbf{f}^{\rm img},\mathbf{f}^{\rm text};P_X) &= \rho\bigl(P_{\mathbf{f}^{\rm img}(X)},\delta_{\mathbf{e}_{\rm Fake}} \bigr),
    \\ \mathcal{V}( \mathbf{f}^{\rm img},\mathbf{f}^{\rm text};Q_X) & = \rho\bigl(Q_{\mathbf{f}^{\rm img}(X)},\delta_{\mathbf{e}_{\rm Real}} \bigr),
    \end{split}
\end{equation*}
here $\rho(\cdot,\cdot)$ is a suitable distance (e.g., Euclidean distance, $1 - \cos(\cdot)$, etc.), and $\delta_{\mathbf{e}_{\rm Fake}}$ and $\delta_{\mathbf{e}_{\rm Real}}$ present the Dirac measures over text anchors $\mathbf{e}_{\rm Fake}$ and ${\mathbf{e}_{\rm Real}}$, respectively.
}
\end{definition}

\begin{definition}[Inter-class Separation]\label{def:clip_inter} 
{\itshape
The separation of CLIP model across $\mathscr{D}_G, \mathscr{D}_N$ is 
\begin{equation}
\mathcal{P}(\mathbf{f}^{\rm img};\mathscr{D}_G,\mathscr{D}_N)\!=\!\min_{P_{X}\in\mathscr{D}_G,Q_{X}\in\mathscr{D}_N}\!\rho\big(P_{\mathbf{f}^{\rm img}(X)},Q_{\mathbf{f}^{\rm img}(X)}\big), \notag
\end{equation}
where $\rho(\cdot,\cdot)$ is a suitable distance defined in Definition \ref{def:clip_intra1}.
}
\end{definition}
A lower $\mathcal{V}_{\textnormal{CLIP}}$ indicates images within the same class are more tightly clustered around their corresponding text anchor, reflecting reduced intra-class variation. 
Conversely, a higher $\mathcal{P}$ suggests greater separation between the clusters of different classes, reflecting enhanced inter-class distinction. 

\begin{definition}[Generalization Error for AI-generated Image Detector]\label{def:clip_err}
{\itshape
Given a detector $\mathbf{f}$ based on the CLIP embedding and trained on training distributions $P_{X_{tr}} \in \mathscr{D}_G$ and $Q_{X_{tr}} \in \mathscr{D}_N$, the generalization error over $\mathscr{D}_G$ and $\mathscr{D}_N$ w.r.t. $\mathbf{f}$ and loss function $\ell: \mathcal{Y}\times \mathcal{Y}\rightarrow \mathbb{R}_{\geq 0}$ is
\begin{equation*}
\begin{split}
\mathrm{err}_{\textnormal{}}(\mathbf{f};\mathscr{D}_G,\mathscr{D}_N) 
=&\max_{P_X\in \mathscr{D}_G} 
\big(\mathbb{E}_{\mathbf{x} \sim P_X} 
\ell\bigl(\mathbf{f}(\mathbf{x}), \text{Fake} \big) 
\\&~~~~~~~~~~- \mathbb{E}_{\mathbf{x} \sim P_{X_{tr}}} 
\ell\bigl(\mathbf{f}(\mathbf{x}), \text{Fake} \big) \big) 
\\
+&\max_{Q_X\in \mathscr{D}_N} 
\big(
\mathbb{E}_{\mathbf{x} \sim Q_X} 
\ell\bigl(\mathbf{f}(\mathbf{x}), \text{Real} \big) 
\\&~~~~~~~~~~- \mathbb{E}_{\mathbf{x} \sim Q_{X_{tr}}} 
\ell\bigl(\mathbf{f}(\mathbf{x}), \text{Real} \big) \big).
\end{split}
\end{equation*}
}
\end{definition}
The generalization error $\mathrm{err}_{\textnormal{}}(\mathbf{f};\mathscr{D}_G,\mathscr{D}_N)$ measures how much the worst-case error on any unseen generative model exceeds that of the training distributions.

\begin{theorem}[Generalization Error Upper Bound]\label{T4.4} If the loss $\ell$ is upper bounded, for a learnable generalization with sufficient inter-class separation defined in Definition \ref{def:clip_inter}, then the generalization error over $\mathscr{D}_G$ and $\mathscr{D}_N$ w.r.t. ${\mathbf{f}}$ (here $\mathbf{f}$ is the detector based on the outputs of $\mathbf{f}^{{\rm img}}$ and $\mathbf{f}^{{\rm text}}$) is
\begin{equation}
    \begin{split}
&\mathrm{err}_{\textnormal{}}(\mathbf{f};\mathscr{D}_G,\mathscr{D}_N) 
\\ \leq & O\Big( \big( \mathcal{V}_{\rm CLIP}^{\rm sup}( \mathbf{f}^{{\rm img}},\mathbf{f}^{{\rm text}};P_{X_{tr}},Q_{X_{tr}} \big)^{\frac{\alpha^2}{(\alpha+d)^2}} \Big),
\end{split}
\end{equation}
for some $\alpha>0$. Here, $d$ denotes the output dimension of $\mathbf{f}^{{\rm img}}$, and
\begin{equation*}
\begin{split}
    &\mathcal{V}_{\rm CLIP}^{\rm sup}( \mathbf{f}^{{\rm img}},\mathbf{f}^{{\rm text}};P_{X_{tr}},Q_{X_{tr}} \big)\\ = &\sup_{\beta\in \mathbb{S}^{d-1}} \mathcal{V}_{\rm CLIP}( \beta^{\top}\mathbf{f}^{{\rm img}},\beta^{\top}\mathbf{f}^{{\rm text}};P_{X_{tr}},Q_{X_{tr}} \big)
    \end{split}
\end{equation*}
is the inter-class variation, here $\mathbb{S}^{d-1}$ is the unit hypersphere defined over $\mathbb{R}^d$.
\end{theorem}

Theorem \ref{T4.4} emphasizes that achieving both high inter-class separation (ensuring distinguishability) and low intra-class variation (minimizing generalization error) is key to attaining generalizable AI-generated image detection.
For more details of the theoretical analysis are provided in \cref{appendix:theory}.

\section{Methodology}
This section describes our approach in two parts:
first, we introduce our discriminative representation learning method,
followed by the multimodal prompt learning.
The overview of our proposed method is shown in \cref{fig:overview}.

\subsection{Discriminative Representation Learning}
To achieve discriminative representation learning, we enhance a supervised contrastive loss \cite{DBLP:conf/nips/KhoslaTWSTIMLK20} by incorporating a multimodal context to encourage tighter clustering of same-class samples.

Given a batch $\{\mathbf{x}_1,...,\mathbf{x}_I\}$ from the training data, we define a multimodal embedding set $\mathcal{H}$ w.r.t. the batch as
\begin{equation*}
    \mathcal{H} = \{\mathbf{h}_{-1},\mathbf{h}_0,\mathbf{h}_1,...,\mathbf{h}_I\}, \quad 
\end{equation*}
where $
    \mathbf{h}_{-1} = \mathbf{e}_{\rm Real}, \mathbf{h}_{0} = \mathbf{e}_{\rm Fake}
$ and $
    \mathbf{h}_i = \mathbf{f}^{\rm img}(\mathbf{x}_i)
$ for any $i>0$. 
We let \(\mathcal{I} = \{-1,0,1,\dots,I\}\) be the corresponding index set. 
For each \(i \in \mathcal{I}\), we define
\begin{align}
    A(i) = \mathcal{I} \setminus \{i\}, ~~~
    P(i) = \bigl\{p \in A(i)\mid y_p = y_i\bigr\}, \notag
\end{align}
where $A(i)$ includes all other indices excluding $i$ itself, and $P(i)$ gathers the {positive} indices that share the same label. 

By incorporating text anchors into \(\mathcal{H}\), we unify vision and language in a single discriminative loss $\mathcal{L}_{\textnormal{dis}}$:
\begin{align}
  - \sum_{i \in \mathcal{I}} 
 \frac{1}{|P(i)|}\sum_{p \in P(i)} 
      \log \Biggl(
      \frac{\exp\bigl(\langle \mathbf{h}_i,\,\mathbf{h}_p\rangle/\tau\bigr)}
           {\displaystyle\sum_{j \in A(i)} \exp\bigl(\langle \mathbf{h}_i,\mathbf{h}_j\rangle /\tau\bigr)}
      \Biggr), 
\label{eq_dis}
\end{align}
where $\mathbf{h}_i, \mathbf{h}_p \in \mathcal{H}$, $\tau$ is a temperature hyperparameter, and $\langle\cdot,\cdot\rangle$ denotes cosine similarity. 

By {pulling} same-class image embeddings together and {attracting} them to their corresponding text anchor, the discriminative loss effectively reduces {intra-class variation}. 
In particular, for an image embedding \(\mathbf{h}_i\) (\(i>0\) with label \(y_i\)), the set \(P(i)\) contains other images of label \(y_i\) as well as the text anchor \(\mathbf{e}_{y_i}\). By minimizing \(\mathcal{L}_{\textnormal{dis}}\), we effectively {maximize} the similarity \(\langle \mathbf{h}_i,\mathbf{h}_p\rangle\) for all \(p \in P(i)\), thereby pulling these embeddings closer in the feature space.

\textbf{Inter-class separation.}
In addition to reducing intra-class variation, inter-class separation naturally arises from the denominator in \cref{eq_dis}. 
% Each embedding \(\mathbf{h}_i \in \mathcal{H}\) aims to {boost} its similarity with positive samples \(\mathbf{h}_p\) (which \(p \in P(i)\)) through the ratio
Each embedding $\mathbf{h}_i \in \mathcal{H}$ attempts to {boost} its similarity with $\mathbf{h}_p$, which is defined as the positive embedding whose index \(p \in P(i)\), via the ratio
\begin{align}
    \frac{\exp\bigl(\langle \mathbf{h}_i,\mathbf{h}_p\rangle/\tau \bigr)}
         {\sum_{j\in A(i)} \exp\bigl(\langle \mathbf{h}_i,\mathbf{h}_j\rangle/\tau \bigr)}, \notag
\end{align}
where the denominator covers {all} other embeddings, including negatives. Minimizing \(\mathcal{L}_{\textnormal{dis}}\) enforces
\begin{equation}
\begin{aligned}
\exp\bigl(\langle \mathbf{h}_i,\mathbf{h}_p\rangle/\tau \bigr) 
&\gg  
\exp\bigl(\langle \mathbf{h}_i,\mathbf{h}_n\rangle/\tau \bigr), 
\end{aligned}
\end{equation}
where $\mathbf{h}_n$ denotes the embedding of a {negative sample} with $y_n \neq y_i$ (i.e., $\mathbf{h}_n$ belongs to a different class).
If any \(\mathbf{h}_n\) exhibits high similarity \(\langle \mathbf{h}_i,\mathbf{h}_n\rangle\), it will shrink this ratio and consequently {increase} the loss. 
Hence, once the loss is minimized, different-class pairs must exhibit lower similarity than same-class pairs, ensuring that \(\langle \mathbf{h}_i,\mathbf{h}_p\rangle\) remains large and \(\langle \mathbf{h}_i,\mathbf{h}_n\rangle\) remains smaller, thereby reducing intra-class variation and increasing inter-class separation.
A more detailed discussion can be seen in \cref{appendix:seperation}.

\textbf{Overall objective.}
% $\{\mathbf{e}_{\rm Real}, \mathbf{e}_{\rm Fake}\}$
% \(y\in \{\text{Real}, \text{Fake}\}\)
For an input image \(\mathbf{x}\) with image embedding 
\(\mathbf{h} = \mathbf{f}^{\rm img}(\mathbf{x})\) 
and text embeddings \(\mathbf{e}_y\) for \(y \in \mathcal{Y}\), 
the predicted probability of label \(\hat{y}\) is
\begin{align*}
{p}(\hat{y}|\mathbf{x})
&=\;
\frac{\exp\bigl(\langle \mathbf{h}, \mathbf{e}_{\hat{y}}\rangle /\tau \bigr)}
     {\displaystyle 
      \sum_{y\in \mathcal{Y}}
      \exp\bigl(\langle \mathbf{h}, \mathbf{e}_y\rangle /\tau \bigr)}.
\end{align*}
Then, the cross-entropy loss is
\begin{align}
\mathcal{L}_\text{ce} 
&= -\sum_{\mathbf{x}\in \mathcal{S}_{{tr}}, \hat{y}\in \mathcal{Y}} \mathbf{1}(\mathbf{x},\hat{y}) \log {p} (\hat{y}|\mathbf{x}),
\end{align}
where $\mathcal{S}_{{tr}}$ is the training data introduced in Section \ref{Sec::Pre} and $\mathbf{1}(\mathbf{x},\hat{y})=1$ if and only if the label of $\mathbf{x}$ is $\hat{y}$; otherwise, $\mathbf{1}(\mathbf{x},\hat{y})=0$.
Finally, our overall objective is
\begin{align}
\min_{\bm{\theta}^{\rm text},\bm{\theta}^{\rm img}}\mathcal{L}
\;=\;
\mathcal{L}_\mathrm{ce}
\;+\;
\alpha\,\mathcal{L}_\mathrm{dis},
\end{align}
where $\bm{\theta}^{\rm text}$ and $\bm{\theta}^{\rm img}$ are the learnable parameters w.r.t. the text encoder $\mathbf{f}^{\rm text}$ and image encoder $\mathbf{f}^{\rm img}$, respectively, and \(\alpha\) is the hyper-parameter to balance the contributions of the cross-entropy loss and the discriminative loss \(\mathcal{L}_\mathrm{dis}\). 

Note that optimizing all parameters in $\mathbf{f}^{\rm text}$ and $\mathbf{f}^{\rm img}$ can be computationally expensive. Therefore, in Section \ref{Sec::MPL}, we will describe how to develop our learnable parameters $\bm{\theta}^{\rm text}$ and $\bm{\theta}^{\rm img}$ to achieve efficient and effective optimization.

\begin{table*}[t]
    \centering
    \caption{
    Comparison of accuracy (\%) between our method and others. All methods were trained on the GenImage SDv1.4 dataset and evaluated across different testing subsets. The best results are highlighted in \textbf{bold}, and the second-best are \underline{underlined}.
    }
    % \resizebox{.95\linewidth}{!}
    {
    \begin{tabular}{lccccccccc}
        \toprule
        {Method} & {Midjourney} & {SDv1.4} & {SDv1.5} & {ADM} & {GLIDE} & {Wukong} & {VQDM} & {BigGAN} & {Avg (\%)} \\
        \midrule
        CNNDet \cite{DBLP:conf/cvpr/WangW0OE20}	& 52.8 	& 96.3 	& \underline{99.5} & 50.1 	& 39.8 	& 78.6 	& 53.4 	& 46.8  & 64.7  \\
        DIRE \cite{DBLP:conf/iccv/WangBZWHCL23} & 50.4 	& \textbf{100.0} & \textbf{99.9} & 52.5 	& 62.7 	& 56.5 	& 52.4 	& 59.5  & 71.2 \\
        UnivFD \cite{DBLP:conf/cvpr/OjhaLL23}   & \textbf{91.5} 	& 96.4 	& 96.1 	& 58.1 	& 73.4 	& {94.5} 	& 67.8 	& 57.7  & 79.4 \\
        CLIPping \cite{DBLP:conf/mir/KhanD24}	& 76.2 	& 93.2 	& 92.8 	& 71.6 	& 87.5 	& 83.3 	& 75.4 	& 75.8  & 82.0 \\
        De-fake \cite{DBLP:conf/ccs/ShaLYZ23}   & 79.9 	& 98.7 	& 98.6 	& 71.6 	& 70.9 	& 78.3 	& 74.4 	& \underline{84.7}  & {84.7} \\
        LaRE \cite{DBLP:conf/cvpr/LuoDYD24}	    & 74.0 	& \textbf{100.0} 	& \textbf{99.9} 	& 61.7 	& 88.5 	& \textbf{100.0} & \textbf{97.2} 	& 68.7  & 86.2 \\
        DRCT \cite{DBLP:conf/icml/ChenZYY24}    & \textbf{91.5} 	& 95.0 	& 94.4 	& \underline{79.4} & \underline{89.2} & {94.7} 	& {90.0} & {81.7} & \underline{89.5} \\
        \midrule
        \rowcolor[HTML]{EFEFEF} \methodName\ (Ours) & \underline{83.2} 	& \underline{98.8} 	& 98.5 	& \textbf{82.7} 	& \textbf{91.3} 	& \underline{97.6} 	& \underline{92.4} 	& \textbf{96.5} & \textbf{92.6} \\
        \bottomrule
    \end{tabular}
    \label{tab:Performance_Comparison_on_GenImage}
    }
\end{table*}

\subsection{Multimodal Prompt Learning}\label{Sec::MPL}
To achieve efficient and effective optimization, we apply multimodal prompt learning that introduces additional learnable embeddings to jointly adapt both visual and textual branches, while freezing original text and image encoders.

\textbf{Deep text prompt learning.}
The text encoder $\mathbf{f}^{\mathrm{text}}$ comprises a word embedding layer $\mathbf{f}_0^{\mathrm{text}}$ followed by $L$ transformer layers $\mathbf{f}_i^{\mathrm{text}}$ for $1 \leq i \leq L$. Given a prompt ${\rm Prompt}$ containing $N$ words, each word ${\rm Prompt}_j$ ($1 \leq j \leq N$) is converted into a $d$-dimensional word embedding by
\[
\mathbf{w}_0^j = \mathbf{f}_0^{\mathrm{text}}({\rm Prompt}_j).
\]
Then, we use $\mathbf{w}_0^j$ to form the initial word embedding matrix
\[
\mathbf{W}_0 = [\,\mathbf{w}_0^1,\,\mathbf{w}_0^2,\dots,\mathbf{w}_0^N\,] 
\;\in\; \mathbb{R}^{d \times N}.
\]
At each transformer layer $\mathbf{f}_i^{\mathrm{text}}$, the word embedding matrix $\mathbf{W}_{i-1}$ from the previous layer is updated as
\[
\mathbf{W}_i = \mathbf{f}_i^{\mathrm{text}}\bigl(\mathbf{W}_{i-1}\bigr).
\]
To facilitate deep text prompt learning, we introduce $B$ additional learnable word embeddings denote as $\bm{\theta}_i = [\bm{\theta}_i^1, \bm{\theta}_i^2, \dots, \bm{\theta}_i^B]\in \mathbb{R}^{d \times B}$ for each transformer layer $\mathbf{f}_i^{\mathrm{text}}$.
The new input at each transformer layer $\mathbf{f}_i^{\mathrm{text}}$ becomes
\begin{align} \label{eq_maple_t}
\mathbf{W}_i = \mathbf{f}_i^{\mathrm{text}}\bigl([\bm{\theta}_{i},\mathbf{W}_{i-1}]\bigr),
\end{align}
where $[\cdot,\cdot]$ denotes concatenation. After processing through all $L$ transformer layers, the final word embedding matrix is $\mathbf{W}_L = [\,\mathbf{w}_L^1,\,\mathbf{w}_L^2,\dots,\mathbf{w}_L^N\,]$, $\mathbf{w}_L^N$ is further projected via a linear layer, denoted as \texttt{TextProj}, to obtain the text embedding $\mathbf{e}$:
\[
\mathbf{e} = \texttt{TextProj}(\mathbf{w}_L^N).
\]
Our learnable parameters $\bm{\theta}^{\rm text}$ are set to $\{\bm{\theta}_i\}_{i=i}^L$, representing the learnable embeddings across all transformer layers. 
While keeping the pre-trained text encoder frozen, our deep prompt learning enables efficient optimization, reduces computational overhead, and effectively tailors the text representation to task-specific contexts.

\textbf{Deep vision prompt learning.}
Similar to the text encoder, the image encoder $\mathbf{f}^{\rm img}$ consists of a patch embedding layer $\mathbf{f}^{\rm img}_0$ and $L$ transformer layers $\mathbf{f}^{\rm img}_i$ for $1 \leq i \leq L$. 
For simplicity, we let the vision and text encoders align the depth for easier coupling.  
We first split an input image $\mathbf{x}$ into $M$ fixed-size patches, each patch ${\rm Patch}_j$ ($1 \leq j \leq M$) is first projected into a $d$-dimensional patch embedding by
\[
\mathbf{z}_0^j=\mathbf{f}^{\rm img}_0({\rm Patch}_j).
\]
Then, we use $\mathbf{z}_0^j$ to form the initial patch embedding matrix
$\mathbf{E}_0 =[\,\mathbf{z}_0^1,\,\mathbf{z}_0^2,\dots,\mathbf{z}_0^M\,] \in \mathbb{R}^{d \times M}$. 
$\mathbf{E}_0$ along with an extra class embedding $\mathbf{c}_0$ are then processed sequentially by the $L$ transformer layers. 
Concretely, at the $i^{\text{th}}$ layer, the previous layer outputs $[\mathbf{c}_{i-1}, \mathbf{E}_{i-1}]$ are passed to the transformer layer $\mathbf{f}^{\rm img}_i$ to yield updated embeddings:
\begin{align}
[\mathbf{c}_i, \mathbf{E}_i] &= \mathbf{f}^{\rm img}_i([\mathbf{c}_{i-1}, \mathbf{E}_{i-1}]). \notag
\end{align}
Following \citet{DBLP:conf/cvpr/KhattakR0KK23}, we argue that prompt learning should {simultaneously} adapt both the vision and language branches for optimal context optimization. 
We achieve multimodal coupling by mapping the learnable word embeddings $\{\bm{\theta}_i\}_{i=i}^L$ into vision embeddings $\tilde{\bm{\theta}}_i$ using a linear mapping function \(\mathcal{F}_i(\cdot)\) with learnable parameters $\bm{\theta}^{\mathcal{F}}_i$,
\[\tilde{\bm{\theta}}_i = \mathcal{F}_i(\bm{\theta}_i;\bm{\theta}^{\mathcal{F}}_i).\]
$\tilde{\bm{\theta}}_i$ are further concatenated with the outputs \([\mathbf{c}_{i-1}, \mathbf{E}_{i-1}]\) from the previous transformer layer.
The new input at each transformer layer $\mathbf{f}_i^{\mathrm{img}}$ becomes
\begin{align} \label{eq_maple_v}
[\mathbf{c}_i, E_i] = \mathbf{f}_i^{\mathrm{img}}([\mathbf{c}_{i-1}, \mathbf{E}_{i-1}, \tilde{\bm{\theta}}_i]). 
\end{align}
After processing through all $L$ transformer layers, 
the final class embedding $\mathbf{c}_{L}$ is projected via a linear layer \texttt{ImageProj} to obtain the image embedding $\mathbf{h}$:
\begin{align}
\mathbf{h} &= \texttt{ImageProj}(\mathbf{c}_{L}). \notag
\end{align}
Our learnable parameters $\bm{\theta}^{\rm img}$ are set to $\{\bm{\theta}^{\mathcal{F}}_i\}_{i=i}^L$, representing the learnable parameters in the mapping function \(\mathcal{F}_i(\cdot)\). 
This explicit mapping $\tilde{\bm{\theta}}_i = \mathcal{F}_i(\bm{\theta}_i;\bm{\theta}^{\mathcal{F}}_i)$ fosters a shared embedding space across both branches, ensuring improved mutual synergy in task-relevant context learning.
By freezing the original encoders and introducing multimodal prompts, our method reduces trainable parameters, preserves CLIP's generalization, and enables joint text-image updates to effectively support discriminative representation learning.

\section{Experiments}
We begin by introducing the datasets and experimental setup, followed by a comparison of \methodName\ with baseline methods. Lastly, we provide additional analyses for further evaluation.

\begin{table*}[t!]
	\caption{
    Performance on the UniversalFakeDetect dataset, evaluated with mean Average Precision (mAP). Methods were trained on ProGAN and tested on various subsets. 
    The best results are highlighted in \textbf{bold}, and the second-best are \underline{underlined}.
    }	\label{tab:Performance_Comparison_on_UniversalFakeDetect_ap}
	{
		\centering
            \setlength{\tabcolsep}{2pt} % 调整列间距
            \renewcommand{\arraystretch}{1.15} % 调整行间距
		\resizebox{1.\linewidth}{!}{
			\begin{tabular}{cc cccccc c cc cc c ccc ccc c c}
				\toprule
				
				\multirow{2}{*}{\shortstack[c]{Detection\\Method}}  & \multicolumn{6}{c}{Generative Adversarial Networks} &\multirow{2}{*}{\shortstack[c]{Deep\\Fakes}} & \multicolumn{2}{c}{Low Level Vision} & \multicolumn{2}{c}{Perceptual Loss} &\multirow{2}{*}{Guided} & \multicolumn{3}{c}{LDM} & \multicolumn{3}{c}{Glide} & \multirow{2}{*}{DALL-E} & Total \\
				\cmidrule(lr){2-7} \cmidrule(lr){9-10} \cmidrule(lr){11-12} \cmidrule(lr){14-16} \cmidrule(lr){17-19} \cmidrule(lr){21-21}

				& \shortstack[c]{Pro-\\GAN} & \shortstack[c]{Cycle-\\GAN} & \shortstack[c]{Big-\\GAN} & \shortstack[c]{Style-\\GAN} & \shortstack[c]{Gau-\\GAN} &  \shortstack[c]{Star-\\GAN}   &  & SITD & SAN & CRN & IMLE & & \shortstack[c]{200\\Steps} & \shortstack[c]{200\\w/ CFG} & \shortstack[c]{100\\Steps} & \shortstack[c]{100\\27} & \shortstack[c]{50\\27} & \shortstack[c]{100\\10} & & \shortstack[c]{mAP (\%)}
				\\ 
				\midrule	
{\shortstack[c]{Spec}}  
& 55.4 & \textbf{100.0} & 75.1 & 55.1 & 66.1 & \textbf{100.0} & 45.2 & 47.5 & 57.1 & 53.6 & 51.0 & 57.7 & 77.7 & 77.3 & 76.5 & 68.6 & 64.6 & 61.9 & 67.8 & 66.2\\
{\shortstack[c]{Patchfor}} 
& 80.9 & 72.8 & 71.7 & 85.8 & 66.0 & 69.3 & 76.6 & 76.2 & 76.3 & 74.5 & 68.5 & 75.0 & 87.1 & 86.7 & 86.4 & 85.4 & 83.7 & 78.4 & 75.7 & 77.7\\
{\shortstack[c]{Co-occurence}}
& \underline{99.7} & 81.0 & 50.6 & 98.6 & 53.1 & 68.0 & 59.1 & 69.0 & 60.4 & 73.1 & 87.2 & 70.2 & 91.2 & 89.0 & 92.4 & 89.3 & 88.4 & 82.8 & 81.0 & 78.1\\
{\shortstack[c]{CNNDet}}  
& \textbf{100.0} & 93.5 & 84.5 & \underline{99.5} & 89.5 & 98.2 & 89.0 & 73.8 & 59.5 & \underline{98.2} & 98.4 & 73.7 & 70.6 & 71.0 & 70.5 & 80.7 & 84.9 & 82.1 & 70.6 & 83.6\\
{\shortstack[c]{DIRE}}
& \textbf{100.0} & 76.7 & 72.8 & 97.1 & 68.4 & \textbf{100.0} & \textbf{98.6} & 54.5 & 65.6 & 97.1 & 93.7 & 94.3 & 95.2 & \underline{95.4} & 95.8 & 96.2 & 97.3 & 97.5 & 68.7 & 87.6\\
{\shortstack[c]{UnivFD}}
& \textbf{100.0} & 99.5 & \underline{99.6} & 97.2 & \textbf{100.0} & {99.6} & 82.5 & 61.3 & \underline{79.0} & 96.7 & \underline{99.0} & 87.8 & \underline{99.1} & 92.2 & \underline{99.2} & 94.7 & 95.3 & 94.6 & 97.2 & 93.4\\
{\shortstack[c]{CLIPping}}
& \textbf{100.0} & \underline{99.9} & 99.4 & \underline{99.5} & \textbf{100.0} & \textbf{100.0} & 92.6 & \underline{81.0} & 72.5 & 91.9 & 98.7 & \textbf{97.4} & 98.9 & 94.3 & 99.1 & \underline{98.9} & \underline{99.3} & \underline{99.0} & \underline{98.9} & \underline{95.9}\\
\hline
\rowcolor[HTML]{EFEFEF}{\shortstack[c]{\methodName\ (Ours)}}
& \textbf{100.0} & \textbf{100.0} & \textbf{99.9} & \textbf{99.8} & \underline{99.9} & \underline{99.9} & \underline{96.0} & \textbf{93.9} & \textbf{84.7} & \textbf{99.9}
& \textbf{100.0} & \underline{96.4} & \textbf{99.9} & \textbf{99.1} & \textbf{99.9} & \textbf{99.8} & \textbf{99.7} & \textbf{99.8} & \textbf{99.9} & \textbf{98.3} \\

				\bottomrule
		\end{tabular}}
	}
\end{table*}

\begin{table*}[t!]
	\caption{
    Performance on the UniversalFakeDetect dataset, evaluated with average accuracy (Avg. Acc). Methods were trained on ProGAN and tested on various subsets. 
    The best results are highlighted in \textbf{bold}, and the second-best are \underline{underlined}.
    }
	\label{tab:Performance_Comparison_on_UniversalFakeDetect_acc}
	{
		\centering
            \setlength{\tabcolsep}{2pt} % 调整列间距
            \renewcommand{\arraystretch}{1.15} % 调整行间距
		\resizebox{1.\linewidth}{!}{
			\begin{tabular}{cc cccccc c cc cc c ccc ccc c c}
				\toprule
				
				\multirow{2}{*}{\shortstack[c]{Detection\\Method}}  & \multicolumn{6}{c}{Generative Adversarial Networks} &\multirow{2}{*}{\shortstack[c]{Deep\\Fakes}} & \multicolumn{2}{c}{Low Level Vision} & \multicolumn{2}{c}{Perceptual Loss} &\multirow{2}{*}{Guided} & \multicolumn{3}{c}{LDM} & \multicolumn{3}{c}{Glide} & \multirow{2}{*}{DALL-E} & Total \\
				\cmidrule(lr){2-7} \cmidrule(lr){9-10} \cmidrule(lr){11-12} \cmidrule(lr){14-16} \cmidrule(lr){17-19} \cmidrule(lr){21-21}

				& \shortstack[c]{Pro-\\GAN} & \shortstack[c]{Cycle-\\GAN} & \shortstack[c]{Big-\\GAN} & \shortstack[c]{Style-\\GAN} & \shortstack[c]{Gau-\\GAN} &  \shortstack[c]{Star-\\GAN}   &  & SITD & SAN & CRN & IMLE & & \shortstack[c]{200\\Steps} & \shortstack[c]{200\\w/ CFG} & \shortstack[c]{100\\Steps} & \shortstack[c]{100\\27} & \shortstack[c]{50\\27} & \shortstack[c]{100\\10} & & \shortstack[c]{Avg.\\Acc (\%)}
				\\ 
				\midrule	
				
{\shortstack[c]{Spec}} 
& 49.9 & \textbf{99.9} & 50.5 & 49.9 & 50.3 & 99.7 & 50.1 & 50.0 & 48.0 & 50.6 & 50.1 & 50.9 & 50.4 & 50.4 & 50.3 & 51.7 & 51.4 & 50.4 & 50.0 & 55.4 \\

{\shortstack[c]{Co-occurence}} 
& 97.7 & 63.2 & 53.8 & {92.5} & 51.1 & 54.7 & 57.1 & 63.1 & 55.9 & 65.7 & 65.8 & 60.5 & 70.7 & 70.6 & 71.0 & 70.3 & 69.6 & 69.9 & 67.6 & 66.9 \\

{\shortstack[c]{CNNDet}}  
& \textbf{100.0} & 85.2 & 70.2 & 85.7 & 79.0 & 91.7 & 53.5 & 66.7 & 48.7 & \underline{86.3} & \underline{86.3} & 60.1 & 54.0 & 55.0 & 54.1 & 60.8 & 63.8 & 65.7 & 55.6 & 69.6 \\

{\shortstack[c]{Patchfor}}  
& 75.0 & 69.0 & 68.5 & 79.2 & 64.2 & 63.9 & 75.5 & \underline{75.1} & \textbf{75.3} & 72.3 & 55.3 & 67.4 & 76.5 & 76.1 & 75.8 & 74.8 & 73.3 & 68.5 & 67.9 & 71.2 \\

{\shortstack[c]{DIRE}} 
& \textbf{100.0} & 67.7 & 64.8 & 83.1 & 65.3 & \textbf{100.0} & \textbf{94.8} & 57.6 & 61.0 & 62.4 & 62.3 & \underline{83.2} & 82.7 & \underline{84.1} & 84.3 & 87.1 & 90.8 & 90.3 & 58.8 & 77.9 \\

{\shortstack[c]{UnivFD}}  
& \textbf{100.0} & \underline{98.5} & \underline{94.5} & 82.0 & \textbf{99.5} & 97.0 & 66.6 & 63.0 & 57.5 & 59.5 & 72.0 & 70.0 & \underline{94.2} & 73.8 & \underline{94.4} & 79.1 & 79.9 & 78.1 & 86.8 & 81.4 \\

{\shortstack[c]{CLIPping}}  
& \underline{99.8} & 95.6 & 93.8 & \textbf{99.2} & 93.4 & \underline{99.2} & 78.5 & 64.4 & 62.8 & 73.3 & 74.4 & \textbf{{84.3}} & 92.8 & 77.5 & 93.3 & \underline{91.2} & \underline{94.4} & \underline{92.0} & \underline{91.5} & \underline{86.9} \\
\hline
\rowcolor[HTML]{EFEFEF}{\shortstack[c]{\methodName\ (Ours)}}  
& \textbf{100.0} & 94.3 & \textbf{96.5} & \underline{96.8} & \underline{93.6} & 96.1 & \underline{88.7} & \textbf{75.8} & \underline{71.9} & \textbf{92.9} & \textbf{92.9} & 82.0 & \textbf{98.3} & \textbf{94.6} & \textbf{98.6} & \textbf{97.5} & \textbf{97.5} & \textbf{98.0} & \textbf{98.6} & \textbf{92.9} \\
				\bottomrule
		\end{tabular}}
	}
\end{table*}

\subsection{Datasets and Experimental Settings}

\textbf{Datasets.}
We evaluate the effectiveness of our proposed method on multiple benchmarks, including {UniversalFakeDetect} \cite{DBLP:conf/cvpr/OjhaLL23} and {GenImage} \cite{DBLP:conf/nips/ZhuCYHLLT0H023}.
Datasets details are provided in \cref{sec:appendix_datasets}.

\textbf{Evaluation metrics.}
Following prior work~\cite{DBLP:conf/cvpr/OjhaLL23,DBLP:conf/nips/ZhuCYHLLT0H023}, we evaluate detection using mean Average Precision (mAP) and classification accuracy. For UniversalFakeDetect, both metrics are reported, while GenImage is evaluated using accuracy with a 0.5  threshold.

\textbf{Baseline methods.} 
We compare \methodName\ with several state-of-the-art AI-generated image detection methods, including Spec \cite{DBLP:conf/wifs/0022KC19}, Co-occurrence \cite{DBLP:conf/mediaforensics/NatarajMMCFBR19}, Patchfor \cite{DBLP:conf/eccv/ChaiBLI20}, CNNDet \cite{DBLP:conf/cvpr/WangW0OE20}, DIRE \cite{DBLP:conf/iccv/WangBZWHCL23}, LaRE \cite{DBLP:conf/cvpr/LuoDYD24}, UnivFD \cite{DBLP:conf/cvpr/OjhaLL23}, CLIPping \cite{DBLP:conf/mir/KhanD24}, De-fake \cite{DBLP:conf/ccs/ShaLYZ23}, and DRCT \cite{DBLP:conf/icml/ChenZYY24}. These methods serve as baselines for evaluating the performance and generalizability of our method.

\textbf{Implementation details.} 
We implement \methodName\ by applying multimodal prompt learning to a pre-trained ViT-L/14 CLIP model. Training is conducted for 10 epochs with a batch size of 128 and a learning rate of 0.002, optimized via SGD on a single NVIDIA L40 GPU.
For the GenImage dataset, we utilize the entire training set comprising 162k real and 162k fake images. For UniversalFakeDetect, similar with \citet{DBLP:conf/mir/KhanD24}, we reduce the training set to 20k real and 20k fake images (out of the original 360k each), as the effect of training data size has been shown to be less pronounced.
To enrich the positive and negative samples in each training batch, we apply a memory bank of size $M$, which stores previously computed features along with their corresponding labels, expanding the sample pool for more effective training. 
Details of the memory bank and all hyperparameter settings are provided in \cref{sec:appendix_bank} and \cref{sec:appendix_imple_details}, respectively.

\subsection{Experimental Results}
\textbf{Comparisons on GenImage.} 
To validate the effectiveness of \methodName, we conducted comparisons using the same experimental protocol as GenImage. All methods were trained on the SDv1.4 subset of GenImage, and results were evaluated across various testing subsets. As shown in \cref{tab:Performance_Comparison_on_GenImage}, most methods achieve high accuracy on diffusion-based subsets such as SDv1.4, SDv1.5, and Wukong. However, a noticeable decline in performance is observed on more challenging subsets like Midjourney, ADM, GLIDE, VQDM, and particularly BigGAN, a non-diffusion-based generator. 
In contrast, \methodName\ demonstrates robust generalizability, showing consistent performance across all subsets. It achieves an average accuracy of 92.6\%, outperforming all baselines. Notably, on BigGAN, \methodName\ boosts accuracy from 84.7\% to 96.5\%, highlighting its ability to handle generative models with diverse architectures. These results validate the effectiveness of \methodName\ in enhancing the generalizability of AI-generated image detection, particularly for unseen and structurally diverse generative models.

\begin{table*}[ht!]
    \centering
    \caption{
    Ablation studies on the GenImage dataset.
    The best results are highlighted in \textbf{bold}, and the second-best are \underline{underlined}.
    }
    \resizebox{1.\linewidth}{!}{
    \begin{tabular}{ccccccccccccc}
        \toprule
        {Baseline} & Multimodal & Discrimitive Loss & Memory Bank & {Midjourney} & {SDv1.4} & {SDv1.5} & {ADM} & {GLIDE} & {Wukong} & {VQDM} & {BigGAN} & {Avg (\%)} \\
        \midrule
        \checkmark & $\times$ & $\times$ & $\times$ & 80.9 &	96.4 &	95.8 &	71.4 &	91.1 &	89.1 &	80.2 &	78.2 &	85.4  \\
        \checkmark & \checkmark & $\times$ & $\times$ & 70.4 &	\textbf{99.7} &	\textbf{99.6} &	71.4 &	86.9 &	\underline{97.2} &	\underline{90.0} &	94.8 &		88.7  \\
        \checkmark & \checkmark & \checkmark & $\times$ & \underline{82.2} &	\underline{99.2} &	\underline{99.1} &	\underline{77.3} &	\textbf{95.6} &	95.7 &	88.2 &	\textbf{98.3} &	\underline{91.9}  \\
        \checkmark & \checkmark & \checkmark & \checkmark & \textbf{83.2} 	& {98.8} 	& 98.5 	& \textbf{82.7} 	& \underline{91.3} 	& \textbf{97.6} 	& \textbf{92.4} 	& \underline{96.5} & \textbf{92.6} \\
        \bottomrule
    \end{tabular}
    }
    \label{tab:ablation}
\end{table*}

\textbf{Comparisons on UniversalFakeDetect.} 
\cref{tab:Performance_Comparison_on_UniversalFakeDetect_ap} and \cref{tab:Performance_Comparison_on_UniversalFakeDetect_acc} present the performance of various methods on the UniversalFakeDetect dataset, evaluated using mAP and average accuracy. These methods achieve near-perfect accuracy on the same generator (i.e., ProGAN), effectively identifying both real and fake images. 
However, their detection performance degrades to varying degrees when tested on other generators. CLIPping leverages prompt learning to optimize CLIP, highlighting its potential for this task. Building upon this, we enhance CLIP further through discriminative representation learning, surpassing all existing methods. 
Specifically, our approach achieves an average accuracy of 92.9\% and an mAP of 98.3\%, showing the effectiveness of using discriminative representation learning to guide multimodal prompt learning, particularly in improving generalizability.

\textbf{Comparisons of generalizability.} 
\begin{table}[t]
    \caption{Cross-dataset evaluation. Best in \textbf{bold}.}
    \centering
    % \resizebox{0.9\linewidth}{!}
    {
    \begin{tabular}{lcccccc}
        \hline
        & \multicolumn{2}{c}{Sora} & \multicolumn{2}{c}{DALL-E 3} & \multicolumn{2}{c}{Infinity} \\
        & Acc & mAP & Acc & mAP & Acc & mAP  \\ \hline
        UnivFD        & 49.8 &	44.2 &	54.8 &	75.4 &	58.4 &	85.5  \\
        CLIPping      & 94.6 &	98.7 &	92.6 &	98.0 &	90.6 &	97.0  \\ \hline
        \rowcolor[HTML]{EFEFEF}\methodName\ (Ours)  & \textbf{95.7} &	\textbf{99.1} &	\textbf{96.7} &	\textbf{99.6} &	\textbf{97.5} &	\textbf{99.6} \\ \hline
    \end{tabular}
    }
    \label{tab:cross_dataset_acc_ap}
    % \vspace{-0.2in}
\end{table}
To showcase \methodName’s ability to generalize, we conduct cross-dataset evaluations. 
As shown in \cref{tab:cross_dataset_acc_ap}, \methodName\ is trained on Stable Diffusion V1.4 and tested on the commercial models Sora \cite{brooks2024video} and DALL-E 3 \cite{betker2023improving}, as well as the emerging AutoRegressive model Infinity \cite{han2024infinityscalingbitwiseautoregressive}.
Following \citet{DBLP:conf/icml/ChenZYY24}, we use MSCOCO \cite{DBLP:conf/eccv/LinMBHPRDZ14} dataset as the real class and use its text descriptions to generate fake images. For each generator, 1000 real and 1000 fake samples are collected.
\methodName\ demonstrates strong performance across these emerging models, validating its robust generalization capability. 

Additional generalizability results, including comprehensive evaluations on the emerging and challenging dataset Chameleon \cite{DBLP:journals/corr/abs-2406-19435}, as well as extensive degradation studies on the GenImage dataset under conditions such as low resolution, JPEG compression, and Gaussian blurring, are detailed in \cref{sec:appendix_additional_experiments}.

\subsection{Ablation Study}
We investigate the impact of the following factors on detection performance: (1) multimodal prompt learning; (2) discriminative loss; and (3) the effect of incorporating a memory bank. The results of the ablation experiments are presented in \cref{tab:ablation}.
We use single-modal prompt learning \cite{DBLP:journals/ijcv/ZhouYLL22} as the baseline method, achieving an accuracy of 85.4\%. Introducing multimodal prompt learning significantly improves the accuracy by 3.3\%, demonstrating the effectiveness of leveraging both vision and language branches to enhance feature alignment and generalizability. Adding the discriminative loss further boosts the accuracy by 3.2\%, indicating its role in minimizing intra-class variation and maximizing inter-class separation, which helps the model learn more robust and distinctive features.
Finally, incorporating a memory bank to enrich the diversity of training samples leads to an additional 0.7\% increase in accuracy, resulting in an overall accuracy of 92.6\%. 
These ablation studies validate the effectiveness of each component in our proposed method and highlight their contributions to the overall performance.

\subsection{Effectiveness of Hyperparameters}
\begin{figure}[t!]
\begin{center}
\centerline{
\includegraphics[width=0.9\columnwidth]{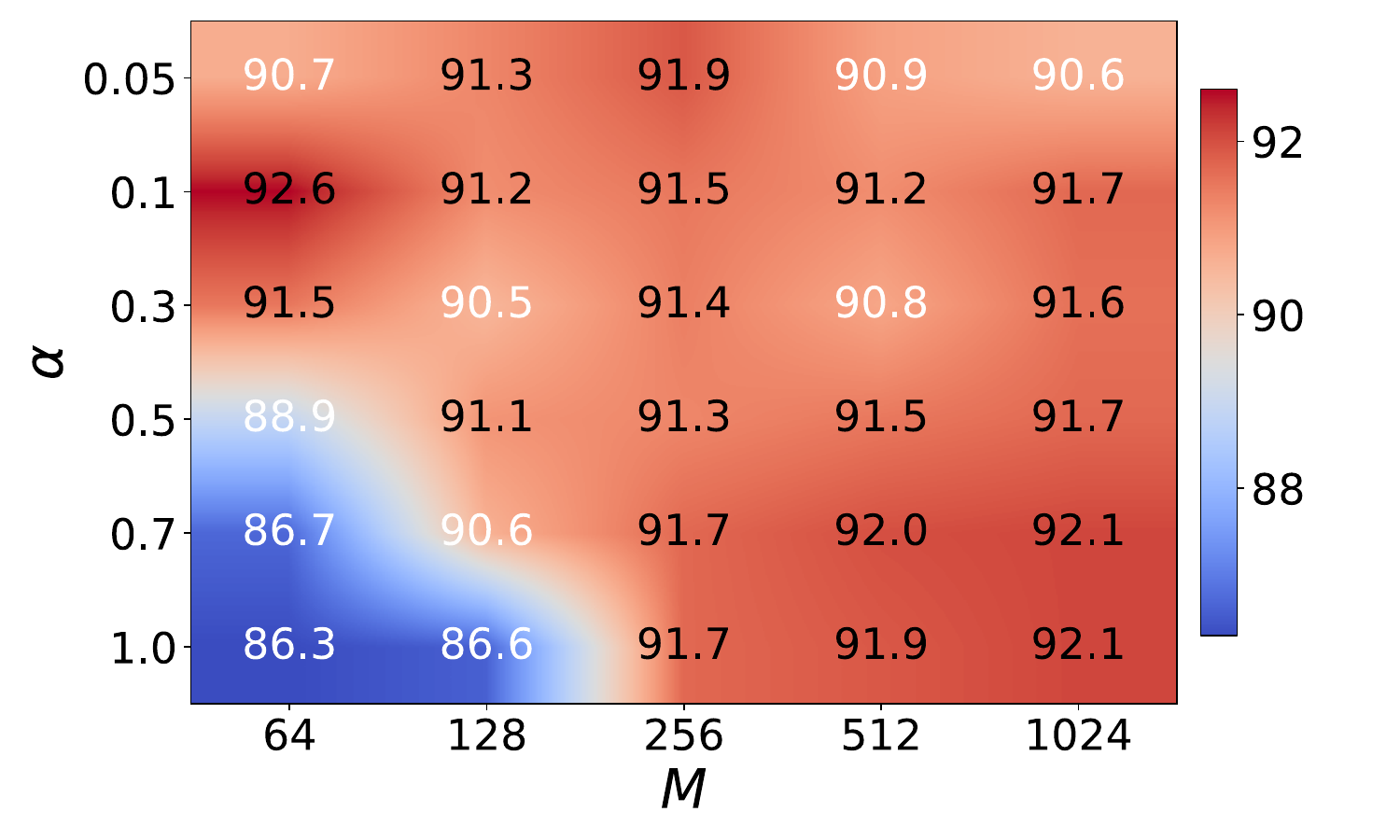}
}
\vspace{-0.1in}
\Description{
The impact of hyperparameters \(\alpha\) and \(M\).
}
\caption{
The impact of hyperparameters \(\alpha\) and \(M\).
}
\vspace{-0.2in}
\label{fig:para}
\end{center}
\end{figure}
We evaluate the impact of the discriminative loss coefficient \(\alpha\) and memory bank size \(M\) on the performance of the GenImage dataset. 
As shown in \cref{fig:para}, \methodName\ achieves stable accuracy across a wide range of hyperparameters, demonstrating its robustness. The best results are obtained with \(\alpha = 0.1\) and \(M = 64\), achieving an accuracy of 92.6\%. 
These values represent an optimal balance between the discriminative loss and the diversity of stored samples in the memory bank, contributing to the observed improvement in accuracy and validating their role in enhancing generalization.

\subsection{Effect of Training Set Size}
\begin{table}[!t]
\centering
\caption{
Effect of training set size on model performance and training time.
We keep an equal amount of real/fake images, e.g., for the 20k subset, we have 10k real and 10k fake images.
}
\label{tab:train_size_effect}
\begin{tabular}{lccc}
\toprule
\textbf{Num. Images} & \textbf{mAP (\%)} & \textbf{Avg. Acc. (\%)} & \textbf{Time (Min.)} \\
\midrule
200k & 97.80 & 92.02 & 203 \\
100k & 97.57 & 91.33 & 132 \\
50k  & 97.54 & 92.28 & 83  \\
20k  & \textbf{98.34} & \textbf{92.87} & 42  \\
10k  & 98.20 & 92.47 & 29  \\
1k   & 96.38 & 89.31 & 16  \\
\bottomrule
\end{tabular}
\end{table}
We further examined the impact of training set size by creating smaller subsets of the UniversalFakeDetect dataset \cite{DBLP:conf/cvpr/OjhaLL23}, each containing 50\% real images and 50\% fake images. Specifically, we prepared datasets with 200k, 100k, 50k, 20k, 10k, and 1k images, and trained our models under these reduced conditions. 
As summarized in \cref{tab:train_size_effect}, the impact of training data size on performance is relatively minor, showing only modest variations in mAP and average accuracy across subsets.
Meanwhile, the training time decreases considerably as the training set size shrinks. 
These findings suggest that even with limited training data, \methodName\ is still possible to achieve robust detection performance without a substantial reduction in generalization capability.

\section{Conclusion}
In this paper, we proposed \methodName, a novel method for enhancing the generalizability of AI-generated image detection. 
By integrating discriminative representation learning and multimodal prompt learning into CLIP, \methodName\ effectively minimizes intra-class variation and maximizes inter-class separation, enabling the model to learn generator-invariant features. 
Comprehensive experiments on multiple benchmarks demonstrate our state-of-the-art performance, showcasing superior adaptability to unseen generators. 
Furthermore, we validated \methodName\ on state-of-the-art generators, highlighting its robustness in handling emerging generative models.

\begin{acks}
This work is supported by the Australian Research Council under Grant
FL190100149.
\end{acks}

\bibliographystyle{ACM-Reference-Format}
\balance
\bibliography{main}

\newpage
\appendix

\section{Related Work} \label{sec:appendix_relatedwork}

\subsection{AI-generated Images Detection}\label{sec:appendix_ai_gen_detect}
In recent years, the rapid advancement of generative models has intensified research on AI-generated image detection, as these models can produce strikingly realistic images that raise concerns over misinformation, privacy, and authenticity. 
Early work often relied on specialized binary classifiers; for instance, CNNDet \cite{DBLP:conf/cvpr/WangW0OE20} directly classifies images as real or fake using a convolutional neural network. 
Several methods focus on frequency-domain analysis to detect inconsistencies: Spec \cite{DBLP:conf/wifs/0022KC19} trains a classifier on the normalized log spectrum of each RGB channel, while LNP \cite{DBLP:conf/eccv/LiuYBXLG22} observes that real images share similar noise patterns in the frequency domain, whereas generated images differ significantly. 
Co-occurrence \cite{DBLP:conf/mediaforensics/NatarajMMCFBR19} inputs co-occurrence matrices into a deep CNN, and GramNet \cite{DBLP:conf/cvpr/LiuQT20} exploits global texture representations to distinguish the substantially different textures of fake images. 
Others emphasize local artifacts rather than global semantics; 
Patchfor \cite{DBLP:conf/eccv/ChaiBLI20} uses classifiers with limited receptive fields to capture local defects, whereas Fusing \cite{DBLP:conf/icip/JuJKXNL22} adopts a dual-branch design combining global spatial information with carefully selected local patches. 
NPR \cite{DBLP:conf/cvpr/TanLZWGLW24} leverages spatial relations among neighboring pixels, and LGrad \cite{DBLP:conf/cvpr/Tan0WGW23} generates gradient maps using a pre-trained CNN, both strategies targeting low-level artifacts. 
AIDE \cite{DBLP:journals/corr/abs-2406-19435} further integrates multiple experts to extract visual artifacts and noise patterns, selecting the highest and lowest-frequency patches to detect AI-generated images based on low-level inconsistencies.

Another line of research focuses on reconstruction-based detection. DIRE \cite{DBLP:conf/iccv/WangBZWHCL23} utilizes the reconstruction capability of diffusion models and trains a classifier on the resulting reconstruction errors. 
LaRE \cite{DBLP:conf/cvpr/LuoDYD24} further refines this direction by using latent-space reconstruction errors as guidance for feature enhancement, specifically targeting diffusion-generated images. 
AEROBLADE \cite{DBLP:conf/cvpr/RickerLF24} adopts a training-free strategy by evaluating autoencoder reconstruction errors within LDMs \cite{DBLP:conf/cvpr/RombachBLEO22}. 
Similarly, DRCT \cite{DBLP:conf/icml/ChenZYY24} generates hard samples by reconstructing real images through a high-quality diffusion model and then applies contrastive learning to capture artifacts.

Recent works have leveraged CLIP-derived features for improved detection, as exemplified by UnivFD \cite{DBLP:conf/cvpr/OjhaLL23}, which trains a classifier in CLIP’s representation space, FAMSeC \cite{DBLP:journals/spl/XuYFLZ25} applies an instance-level, vision-only contrastive objective, and CLIPping \cite{DBLP:conf/mir/KhanD24}, which applies prompt learning and linear probing on CLIP’s encoders. 
While these methods show promise, they still struggle to generalize to unseen models, and focusing on a single modality in CLIP can be suboptimal. 
To address these issues, we propose a method that simultaneously optimizes image and text features using discriminative representation learning, thereby capturing generator-agnostic characteristics and enhancing generalization.

\subsection{Pre-trained Vision-Language Models} \label{sec:appendix_vl_models}
In recent years, large-scale pre-trained models that integrate both image and language modalities have achieved remarkable success, demonstrating robust performance across a variety of tasks \cite{DBLP:journals/pami/ZhangHJL24}. 
These models attract attention for their strong zero-shot capabilities and robustness to distribution shifts. 
Among them, Contrastive Language–Image Pretraining (CLIP) \cite{DBLP:conf/icml/RadfordKHRGASAM21} stands out as a large-scale approach exhibiting exceptional zero-shot performance on tasks such as image classification \cite{DBLP:journals/tfs/ShiLFZ24, DBLP:conf/mm/WangZ0024, DBLP:conf/ijcnn/WangZFL24} and image-text retrieval \cite{DBLP:conf/mm/MaXSYZJ22}. 
CLIP is trained on a dataset of 400 million image-text pairs using a contrastive loss that maximizes similarity between matched pairs while minimizing similarity between mismatched pairs.

Although CLIP demonstrates impressive zero-shot performance, further fine-tuning is often required to reach state-of-the-art accuracy on specific downstream tasks. For instance, on the simple MNIST dataset \cite{DBLP:journals/spm/Deng12}, the zero-shot CLIP model (ViT-B/16) achieved only 55\% accuracy. 
However, fully fine-tuning CLIP on a downstream dataset compromises its robustness to distribution shifts \cite{DBLP:conf/cvpr/WortsmanIKLKRLH22}. 
To address this issue, numerous studies have proposed specialized fine-tuning strategies for CLIP. One example is CoOp \cite{DBLP:journals/ijcv/ZhouYLL22}, which injects learnable vectors into the textual prompt context and optimizes these vectors during fine-tuning while freezing CLIP’s vision and text encoders. 
Nevertheless, focusing solely on the text branch may lead to suboptimal performance. Consequently, MaPLe \cite{DBLP:conf/cvpr/KhattakR0KK23} extends prompt learning to both the vision and language branches, thereby enhancing alignment between these representations. 
Building on MaPLe’s approach, we incorporate our discriminative representation learning on multimodal to address generalization challenges in AI-generated image detection.

\section{Theoretical Analysis} \label{appendix:theory}
\subsection{Minimizing Variation for Enhanced Generalization}
\newcommand{\Eall}{\Ecal_\textnormal{all}}
\newcommand{\Eava}{\Ecal_\textnormal{avail}}

\citet{DBLP:conf/nips/YeXCLLW21} provide generalization error bounds based on the {notion} of variation. Therefore, controlling the intra-class variation is crucial for bounding the generalization error. For completeness, we adapt the results from \citet{DBLP:conf/nips/YeXCLLW21} to our multimodal setting, deriving upper bounds on the generalization error.

\begin{definition}[Intra-class Variation]\label{def_invariance_general}
{\itshape
Let $\mathcal{V}\bigl(\mathbf{f}^{\rm img}; \mathscr{D}_G\bigr)$ denote the intra-class variation of a feature extractor $\mathbf{f}^{\rm img}$ on the set of generated image distributions $\mathscr{D}_G$, and let $\mathcal{V}\bigl(\mathbf{f}^{\rm img}; \mathscr{D}_N\bigr)$ denote the variation on the set of natural distributions $\mathscr{D}_N$. We define
\begin{equation}
\label{eq:variation_general}
\begin{split} 
&\mathcal{V}\bigl(\mathbf{f}^{\rm img};\,\mathscr{D}_G,\mathscr{D}_N\bigr) 
= \max \Bigl\{\mathcal{V}\bigl(\mathbf{f}^{\rm img};\mathscr{D}_G\bigr), \mathcal{V}\bigl(\mathbf{f}^{\rm img};\,\mathscr{D}_N\bigr)\Bigr\},
\end{split}
\end{equation}
where
\begin{align}
\mathcal{V}\bigl(\mathbf{f}^{\rm img};\mathscr{D}_G\bigr)
=
\max_{{P,\tilde{P}}\in\mathscr{D}_G}\rho\Bigl(P_{\mathbf{f}^{\rm img}(X)},\tilde{P}_{\mathbf{f}^{\rm img}(X)}\Bigr), \\
\mathcal{V}\bigl(\mathbf{f}^{\rm img};\,\mathscr{D}_N\bigr)
=
\max_{{Q,\tilde{Q}}\in\mathscr{D}_N}\rho\Bigl(Q_{\mathbf{f}^{\rm img}(X)},\tilde{Q}_{\mathbf{f}^{\rm img}(X)}\Bigr).
\end{align}

}
\end{definition}

\begin{lemma}[Variation Bound]\label{lemma_variation_bound}
{\itshape
For any feature extractor $\mathbf{f}^{\rm img}$, the intra-class variation satisfies
\begin{align}
\label{eq:variation_bound}
\mathcal{V}\bigl(\mathbf{f}^{\rm img};\,\mathscr{D}_G,\mathscr{D}_N\bigr) 
\;\;\le\;\;
2 \max_{P_X\in\mathscr{D}_G,\,Q_X\in\mathscr{D}_N} \quad & \nonumber \\
\mathcal{V}_{\rm CLIP}\!\Bigl(\mathbf{f}^{\rm img},\,\mathbf{f}^{\rm text};\,P_X,\,Q_X\Bigr),
\end{align}
where the right-hand side $\mathcal{V}_{\rm CLIP}$ is the variation measured via CLIP-based text anchors.
}
\end{lemma}

\begin{proof} 
We first show that for any feature extractor $\mathbf{f}^{\rm img}$,
\begin{equation}
\begin{split}
\max_{P,\tilde{P}\in\mathscr{D}_G}
\rho\Bigl(P_{\mathbf{f}^{\rm img}(X)},\tilde{P}_{\mathbf{f}^{\rm img}(X)}\Bigr)
&\le
\max_{P\in\mathscr{D}_G}
\rho\bigl(P_{\mathbf{f}^{\rm img}(X)},\delta_{\mathbf{e}_{\rm Fake}}\bigr)
\\ &\quad +
\max_{\tilde{P}\in\mathscr{D}_G}
\rho\bigl(\tilde{P}_{\mathbf{f}^{\rm img}(X)},\delta_{\mathbf{e}_{\rm Fake}}\bigr),
\\
\max_{Q,\tilde{Q}\in\mathscr{D}_N}
\rho\Bigl(Q_{\mathbf{f}^{\rm img}(X)},\tilde{Q}_{\mathbf{f}^{\rm img}(X)}\Bigr)
&\le
\max_{Q\in\mathscr{D}_N}
\rho\bigl(Q_{\mathbf{f}^{\rm img}(X)},\delta_{\mathbf{e}_{\rm Real}}\bigr)
\\ &\quad +
\max_{\tilde{Q}\in\mathscr{D}_N}
\rho\bigl(\tilde{Q}_{\mathbf{f}^{\rm img}(X)},\delta_{\mathbf{e}_{\rm Real}}\bigr).
\end{split}
\end{equation}
These inequalities imply that summing the worst-case deviation for two distributions (\(P\) and \(\tilde{P}\)) can be upper-bounded by the sum of deviations from each distribution to a text anchor (the Dirac measure $\delta_{\mathbf{e}_{\rm Fake}}$ or $\delta_{\mathbf{e}_{\rm Real}}$). Consequently,
\begin{align}
\max_{P,\tilde{P}\in\mathscr{D}_G}
\rho\Bigl(P_{\mathbf{f}^{\rm img}(X)},\tilde{P}_{\mathbf{f}^{\rm img}(X)}\Bigr)
+
\max_{Q,\tilde{Q}\in\mathscr{D}_N}
\rho\Bigl(Q_{\mathbf{f}^{\rm img}(X)},\tilde{Q}_{\mathbf{f}^{\rm img}(X)}\Bigr)
&\le
\notag \\
2 \times
\max_{P_X\in\mathscr{D}_G,\,Q_X\in\mathscr{D}_N}
\mathcal{V}_{\rm CLIP}\Bigl(\mathbf{f}^{\rm img},\mathbf{f}^{\rm text};P_X,Q_X\Bigr).
\end{align}
By the definition of \(\mathcal{V}\bigl(\mathbf{f}^{\rm img};\,\mathscr{D}_G,\mathscr{D}_N\bigr)\), we then conclude the bound in \cref{lemma_variation_bound}.
\end{proof}

\begin{definition}[Expansion Function \cite{DBLP:conf/nips/YeXCLLW21}] \label{def_expan}
{\itshape
We say a function $s:\mathbb R^+ \cup \{0\} \to \mathbb R^+ \cup \{0, +\infty\}$ is an expansion function, iff the following properties hold: 
1) $s(\cdot)$ is monotonically increasing and $s(x)\geq x,\forall x\geq0$; 2) $\lim_{x\to 0^+} s(x) = s(0) = 0$.
}
\end{definition}

Since it is impossible to generalize to an arbitrary distribution, characterizing the relationship between $P_{X_{tr}}$ and $P_X$, as well as between $Q_{X_{tr}}$ and $Q_X$ is essential to formalize generalization. Building on the expansion function, we define the {learnability} of a generalization problem as follows:

\begin{definition}[Learnability]\label{def_learn}
{\itshape
Let $\Phi$ be the feature space. We say a generalization problem from $P_{X_{tr}}, Q_{X_{tr}}$ to $P_X, Q_X$ is {learnable} if there exist an expansion function $s(\cdot)$ and a constant $\delta\ge 0$ such that for all $\mathbf{f}^{\rm img}(\mathbf x) \in \Phi$ satisfying $\mathcal{P}(\mathbf{f}^{\rm img};\mathscr{D}_G, \mathscr{D}_N) \geq \delta$, the following hold:
\begin{align}
s\bigl(\mathcal{V}(\mathbf{f}^{\rm img},\mathbf{f}^{\rm text};P_{X_{tr}})\bigr) 
&\;\ge\; 
\mathcal{V}(\mathbf{f}^{\rm img},\mathbf{f}^{\rm text};P_X), \\
s\bigl(\mathcal{V}(\mathbf{f}^{\rm img},\mathbf{f}^{\rm text};Q_{X_{tr}})\bigr)
&\;\ge\;
\mathcal{V}(\mathbf{f}^{\rm img},\mathbf{f}^{\rm text};Q_X).
\end{align}
If such $s(\cdot)$ and $\delta$ exist, we further call this problem $(s(\cdot), \delta)$-learnable.
}
\end{definition}

\begin{theorem}[Error Upper Bound]\label{general bound full}
Suppose we have learned a classifier with loss function $\ell(\cdot, \cdot)$, and for all $y \in \mathcal{V}$, the conditional density $p_{h|Y}(h|y)$ satisfies $p_{h|Y} (h|y) \in L^2(\mathbb{R}^D).$  
Let $\mathbf{f}^{\rm img} \in \mathbb{R}^D$ denote the image feature extractor, and define the characteristic function of the random variable $h|Y$ as $\hat{p}_{h|Y}(t|y) = \mathbb{E}[\exp\{i \langle t, h \rangle\} \mid Y = y].$

Assume the hypothesis space $\mathcal{F}$ satisfies the following regularity conditions: there exist constants $\alpha, M_1, M_2 > 0$ such that for all $f \in \mathcal{F}$ and $y \in \mathcal{Y}$,
\begin{align}
\int_{h \in \mathbb{R}^D} p_{h|Y}(h|y) |h|^\alpha \mathrm{d}h \leq M_1, \quad 
\int_{t \in \mathbb{R}^D} |\hat{p}_{h|Y}(t|y)| |t|^\alpha \mathrm{d}t \leq M_2.
\end{align}

If $(\mathbf{f}; \mathscr{D}_G, \mathscr{D}_N)$ is $(s(\cdot), \delta)$-learnable under $\Phi$ with Total Variation $\rho$\footnote{For two distributions $\mathbb{P}, \mathbb{Q}$ with probability density functions $p, q$, $\rho(\mathbb{P}, \mathbb{Q}) = \frac{1}{2} \int_x |p(x) - q(x)| \mathrm{d}x$.}, then the generalization error is bounded as:
\begin{equation} 
\mathrm{err}(\mathbf{f}; \mathscr{D}_G, \mathscr{D}_N) \leq 
O\Big( \big( \mathcal{V}_{\rm CLIP}^{\rm sup}(\mathbf{f}^{\rm img}, \mathbf{f}^{\rm text}; P_{X_{tr}}, Q_{X_{tr}}) \big)^{\frac{\alpha^2}{(\alpha + D)^2}} \Big),
\end{equation}
where the constant $O(\cdot)$ depends on $D$, $\alpha$, $M_1$, and $M_2$.
\end{theorem}
\begin{proof}
Given distributions $P_{\alpha}$ and $\Delta_{\alpha}$ defined over \[\mathcal{X}\;\times\;\{{\;\rm Fake},\;{\rm Real\;}\}\] satisfying that
\begin{equation}
\begin{split}
P(\mathbf{x}|y={\rm Fake})=P_{X_{tr}}(\mathbf{x}),\quad
P(\mathbf{x}|y={\rm Real})=Q_{X_{tr}}(\mathbf{x}), \\
\Delta(\mathbf{x}|y={\rm Fake})=\delta_{\rm Fake}(\mathbf{x}),\quad
\Delta(\mathbf{x}|y={\rm Real})=\delta_{\rm Real}(\mathbf{x}),
\end{split}
\end{equation}
and
\begin{equation}
\begin{split}
P(y={\rm Fake})=\alpha,\quad
P(y={\rm Real})=1-\alpha, \\
\Delta(y={\rm Fake})=\alpha,\quad
\Delta(y={\rm Real})=1-\alpha,
\end{split}
\end{equation}
we set $\mathcal{E}_{avail}$ in Theorem~4.1 of \citet{DBLP:conf/nips/YeXCLLW21} as $\{P_{\alpha}:\forall \alpha\in(0,1)\}\cup \{\Delta_{\alpha}:\forall \alpha\in(0,1)\}$.
Then this result can be concluded by Theorem~4.1 of \citet{DBLP:conf/nips/YeXCLLW21} and our \cref{lemma_variation_bound} directly.
\end{proof}

\section{Inter-class Separation} \label{appendix:seperation}
Analysis in \cite{DBLP:conf/nips/XieW0C023} shows that inter-class dispersion is strongly correlated with the model accuracy, reflecting the generalization performance on test data.

\subsection{Where Inter-class Separation Comes From}
\label{sec:appendix_inter_class_sep}

Recall that \(\mathcal{L}_{\mathrm{dis}}\) involves the denominator
\[
\sum_{j \in A(i)} \exp\bigl(\langle \mathbf{h}_i,\mathbf{h}_j\rangle/\tau\bigr),
\]
which sums over {all} other samples \(a\in A(i)\) (both positives and negatives). The goal of minimizing 
\[
\log\Bigl(
  \frac{\exp\bigl(\langle \mathbf{h}_i,\mathbf{h}_p\rangle/\tau \bigr)}
       {\sum_{j\in A(i)}\exp\bigl(\langle \mathbf{h}_i,\mathbf{h}_j\rangle/\tau \bigr)}
\Bigr)
\]
is to make the positive pair \(\langle \mathbf{h}_i,\mathbf{h}_p\rangle\) {dominate} that ratio. 
For any negative \(n\) (with \(y_n \neq y_i\)), having a high similarity \(\langle \mathbf{h}_i,\mathbf{h}_n\rangle\) would reduce the fraction in the softmax, thereby raising the loss. 
Hence, the optimization naturally favors
\begin{align}
\exp\bigl(\langle \mathbf{h}_i,\mathbf{h}_p\rangle/\tau\bigr) 
\;\;\gg\;\;
\exp\bigl(\langle \mathbf{h}_i,\mathbf{h}_n\rangle/\tau\bigr)
\quad & \notag\\
\forall \; p\in P(i),\; n\in A(i)\setminus P(i). & \notag
\end{align}
This condition simultaneously {pulls} same-class pairs closer and {pushes} different-class pairs apart, thus increasing inter-class separation. 
Intuitively, if two samples belong to different classes but still have a large dot product, they “compete” with the positive pairs, causing a higher loss. 
Over many gradient steps, the model adapts by reducing the similarity between different-class samples.

\subsection{Mathematical Basis for Inter-class Separation}
\label{sec:appendix_formal_argument}

We can further formalize the above intuition by analyzing \(\mathcal{L}_{\mathrm{dis}}\) from a {pairwise} and {margin-based} perspective. 

\textbf{Pairwise comparisons with log-softmax.}
Rewrite the loss \(\mathcal{L}_{\mathrm{dis}}\) for each $i \in \mathcal{I}$ as:
\[
- \sum_{p \in P(i)} 
   \log\Bigl(
   \frac{e^{\langle \mathbf{h}_i, \mathbf{h}_p\rangle/\tau}}
        {\sum_{j\in A(i)} e^{\langle \mathbf{h}_i,\mathbf{h}_j\rangle/\tau}}
   \Bigr)
   \times \frac{1}{|P(i)|}.
\]
For this expression to be small, each term inside the log must be large, {i.e.}, 
\[
\frac{e^{\langle \mathbf{h}_i, \mathbf{h}_p\rangle/\tau}}
     {\sum_{j\in A(i)} e^{\langle \mathbf{h}_i,\mathbf{h}_j\rangle/\tau}}
\;\;\text{is close to 1.}
\]
As a result, any negative \(n\neq p\) with high similarity \(\langle \mathbf{h}_i,\mathbf{h}_n\rangle\) {directly} lowers this probability and thus {increases} the loss. Minimizing the sum effectively {penalizes} large similarities to negatives. 
In short, {raising} \(\langle \mathbf{h}_i,\mathbf{h}_p\rangle\) forces \(\langle \mathbf{h}_i,\mathbf{h}_n\rangle\) (for \(n\neq p\)) to stay lower.

\textbf{Margin-based constraints.}
Consider enforcing a margin \(\delta>0\) between positive and negative similarities, such that
\[
\langle \mathbf{h}_i,\mathbf{h}_p\rangle 
\;\ge\; 
\langle \mathbf{h}_i,\mathbf{h}_n\rangle + \delta,
\quad
\forall (p\in P(i),\; n \notin P(i)).
\]
When substituted into the softmax term, even a small positive margin \(\delta\) significantly reduces the negative pairs’ exponential scores relative to the positive pairs. Minimizing the overall loss under such a margin constraint reveals a {pairwise repulsion} effect: 
(1) If \(\langle \mathbf{h}_i,\mathbf{h}_p\rangle\) is consistently larger than \(\langle \mathbf{h}_i,\mathbf{h}_n\rangle\) by at least \(\delta\), then the ratio for each positive sample \(p\) stays high.  
(2) Violating this margin (letting \(\langle \mathbf{h}_i,\mathbf{h}_n\rangle\) get too close or exceed \(\langle \mathbf{h}_i,\mathbf{h}_p\rangle\)) incurs a heavier penalty, pushing the model to further lower negative similarities.  

\textbf{Why this promotes inter-class separation.}
Since any two samples \(i\) and \(n\) with different labels eventually appear in each other’s denominators, repeated updates across the entire dataset ensure that \(\mathbf{h}_n\) and \(\mathbf{h}_i\) do not remain highly similar if \(y_n \neq y_i\). Over time, the network learns a {global} arrangement in which inter-class pairs are systematically pushed apart, producing well-separated clusters in the embedding space.

\subsection{Geometric Interpretation}
\label{sec:appendix_geometry}
Discriminative representation learning can also be viewed through a purely geometric lens. Each feature vector \(\mathbf{h}_i\) is normalized (often to lie on the unit hypersphere), and the learning objective \(\mathcal{L}_{\mathrm{dis}}\) penalizes large angles (low cosine similarity) for same-class pairs while rewarding large angles for different-class pairs. 

Concretely, the supervised contrastive loss exhibits the following geometric effects:

{Same-class:}
For positive pairs \((i, p)\) where \(p \in P(i)\) and shares the label \(y_i\), the dot product \(\langle \mathbf{h}_i, \mathbf{h}_p \rangle\) should be high. Since the features are normalized to lie on a unit hypersphere, this implies \(\mathbf{h}_p\) is positioned within a narrow cone centered on \(\mathbf{h}_i\).

{Cross-class:}
For negative pairs \((i, n)\) where \(n \notin P(i)\) and \(y_n \neq y_i\), the dot product \(\langle \mathbf{h}_i, \mathbf{h}_n \rangle\) should be comparatively low. Geometrically, this ensures that \(\mathbf{h}_n\) is directed away from \(\mathbf{h}_i\) on the sphere, increasing the angular separation between different-class samples.

Over the dataset, these pairwise “push-pull” forces yield a partitioning of the hypersphere into well-separated clusters. The higher-level geometry of the loss function ensures that each class cluster remains cohesive while classes themselves lie farther apart.

\subsection{Role of Text Anchors}
\label{sec:appendix_text_anchors}

In multi-modal contexts, each class can also have a dedicated {text anchor}, \(\mathbf{e}_y\). Below, we elaborate on how these anchors reinforce the separation effect:

1. Explicit Class Centers. 
   Each text embedding \(\mathbf{e}_{\mathrm{Real}}\) or \(\mathbf{e}_{\mathrm{Fake}}\) (for example) acts as a fixed or learnable “center” of that class. Images with label \(\mathrm{Real}\) are {pulled} toward \(\mathbf{e}_{\mathrm{Real}}\), and images with label \(\mathrm{Fake}\) are {pulled} toward \(\mathbf{e}_{\mathrm{Fake}}\).  

2. Cross-text Repulsion.
   Text anchors for different classes, say \(\mathbf{e}_{\mathrm{Real}}\) vs.\ \(\mathbf{e}_{\mathrm{Fake}}\), serve as negatives to each other. Consequently, the system learns to keep these class-centered vectors well apart in the embedding space, reinforcing the boundary between classes, thus further intensifying discriminative separation.

3. Strong Guidance for Images.
   Because each image embedding with label $y_i$ sees \(\mathbf{e}_{y_i}\) as its top positive match, it gains a clear “target direction” on the sphere, ensuring that all \(\mathrm{Real}\) images converge near \(\mathbf{e}_{\mathrm{Real}}\) and all \(\mathrm{Fake}\) images converge near \(\mathbf{e}_{\mathrm{Fake}}\). If there are multiple classes, the same logic extends to each label’s text anchor.

In summary, text anchors serve as pivotal reference points that shape the global arrangement of class embeddings. Anchors from different classes introduce a mutual repulsion, promoting inter-class separation, while each anchor and its associated images maintain attractive forces that consolidate intra-class structure. By coupling language and vision through these textual anchors, the framework not only integrates information from both modalities but also rigorously enforces class boundaries in the embedding space.

\section{Extension with Memory Bank.} \label{sec:appendix_bank}
\begin{figure*}[t!]
\centering{
\includegraphics[page=2, width=0.85\textwidth]{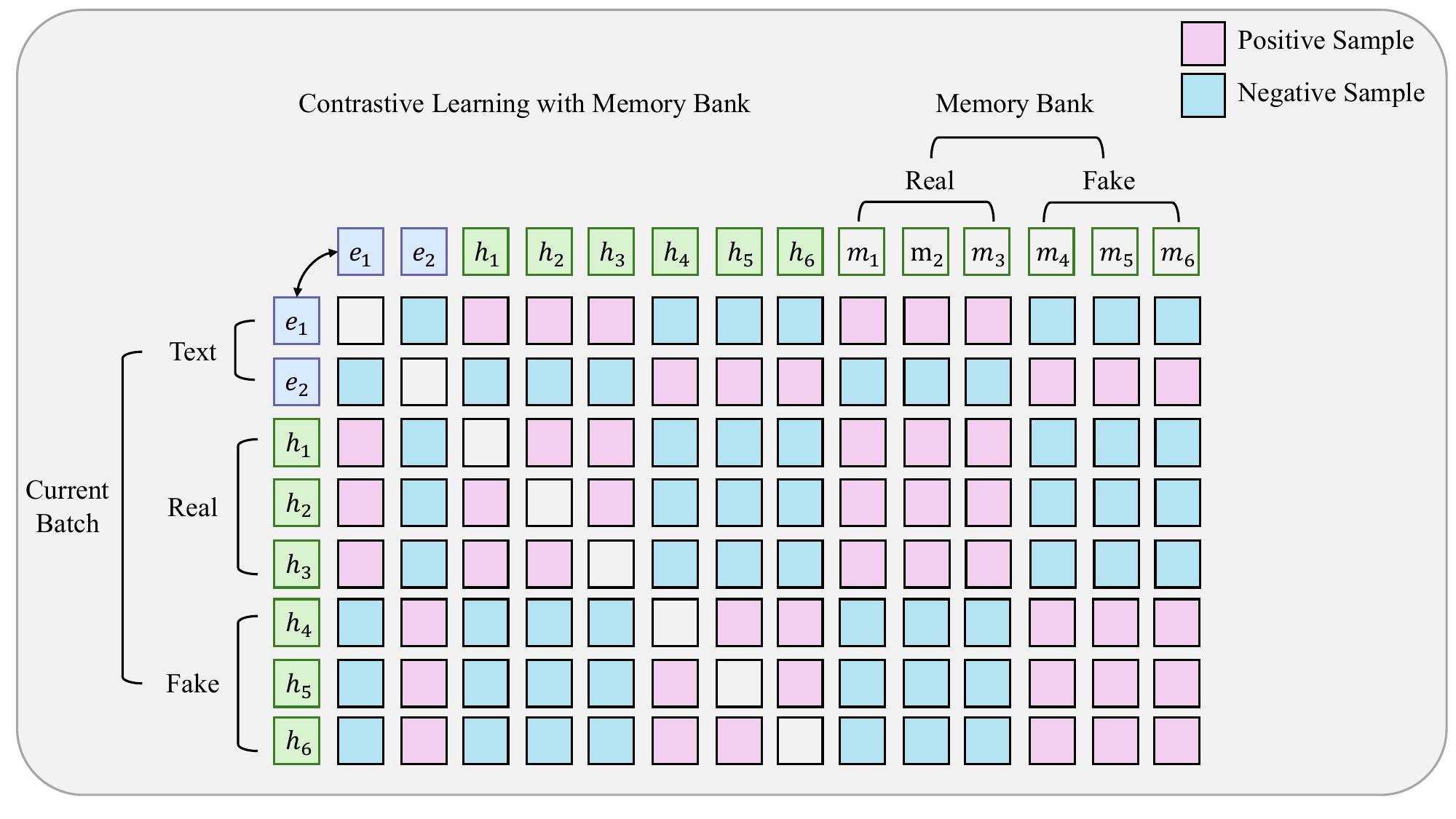}
}
\Description{
Overview of the memory bank.
}
\caption{
Overview of the memory bank.
During training, the memory bank maintains a dynamic queue of historical embeddings and their labels. For each batch, we concatenate current batch embeddings with memory bank samples to construct an augmented embedding set. This expanded pool enables richer positive and negative samples while preserving temporal diversity. The memory bank is updated in a first-in-first-out (FIFO) manner with current batch embeddings after each training step.
}
\label{fig:bank}
\end{figure*}
To further enhance both positive and negative sample diversity while maintaining temporal consistency, we introduce a memory bank mechanism that stores historical embeddings across training iterations. As illustrated in Figure~\ref{fig:bank}, this module operates through three key phases: (1) augmented embedding construction, (2) discriminative representation learning with expanded sample pools, and (3) dynamic memory updating.

\textbf{Augmented embedding construction.}
Let \(\mathcal{M} = \{\mathbf{m}_k\}_{k=1}^M\) denote the memory bank storing \(M\) historical embeddings and their corresponding labels. For each training batch \(\mathcal{H} = \{\mathbf{h}_i\}_{i={-1}}^I\), we construct an augmented embedding set:
\begin{equation}
    \tilde{\mathcal{H}} = \mathcal{H} \cup \{\mathbf{m}_k\}_{k=1}^M.
\end{equation}
This expansion enables each anchor to observe \(M\) additional historical examples while maintaining computational efficiency.

\textbf{Dircriminative representation learning with expanded pools.}
We extend the original discriminative loss in Eq.~\eqref{eq_dis} by recalculating the positive relationships over \(\tilde{\mathcal{H}}\):
\begin{align}
    \tilde{A}(i) &= A(i) \cup \{I+1,\dots,I+M\}, \\
    \tilde{P}(i) &= \bigl\{p \in \tilde{A}(i)\mid y_p = y_i\bigr\}.
\end{align}
The revised memory-augmented contrastive loss becomes:
\begin{align}
\mathcal{L}_{\text{dis}}^{\prime} = -\sum_{i \in \tilde{\mathcal{I}}} \frac{1}{|\tilde{P}(i)|}\sum_{p \in \tilde{P}(i)} \log\frac{\exp(\langle\mathbf{h}_i,\mathbf{h}_p\rangle/\tau)}{\sum_{j \in \tilde{A}(i)}\exp(\langle\mathbf{h}_i,\mathbf{h}_j\rangle/\tau)},
\end{align}
where \(\tilde{\mathcal{I}} = \{-1,\dots,I+M\}\) indexes the augmented set. This formulation forces each anchor to discriminate against both current and historical negative samples while aggregating positives across temporal domains. 
Note that historical embeddings in the current batch are detached from the computational graph and do not receive gradient updates.

\textbf{Memory update strategy.}
We employ a first-in-first-out (FIFO) update rule after processing each batch:
\begin{equation}
    \mathcal{M} \leftarrow \mathcal{M}_{\setminus \{1,\dots,I\}} \cup \{\mathbf{h}_i\}_{i=1}^I,
\end{equation}
where \(I\) denotes the batch size. This ensures the memory bank retains recent embeddings while preserving diversity through gradual replacement of older samples.

Through synergistic integration of historical and current embeddings, our memory bank enables learning from a more comprehensive distribution of positive and negative samples. This strengthens both intra-class compactness and inter-class separation, particularly benefiting generalization to unseen generative models.

\textbf{Computational Overhead of the Memory Bank.} 
The memory bank introduces negligible computational overhead and is virtually cost-free. 
It reuses historical embeddings from previous batches without requiring extra forward passes, thereby avoiding any additional encoding.
Each stored embedding is a low-dimensional vector (dimension 1024), so even with a typical bank size ranging from \(M=64\) to \(M=1024\), the total GPU memory overhead is at most 2 MB using float16 precision, insignificant compared to the overall GPU memory cost during training. 
Moreover, during loss computation, memory bank embeddings are detached from the computational graph and do not receive gradients, limiting the extra computation to lightweight matrix operations. 
Overall, the memory bank adds minimal training cost while consistently enhancing performance, making it an efficient and practical component.

\section{Multimodal Prompt Learning} \label{sec:appendix_maple}
\begin{figure*}[t!]
\centering{
\includegraphics[page=3, width=0.85\textwidth]{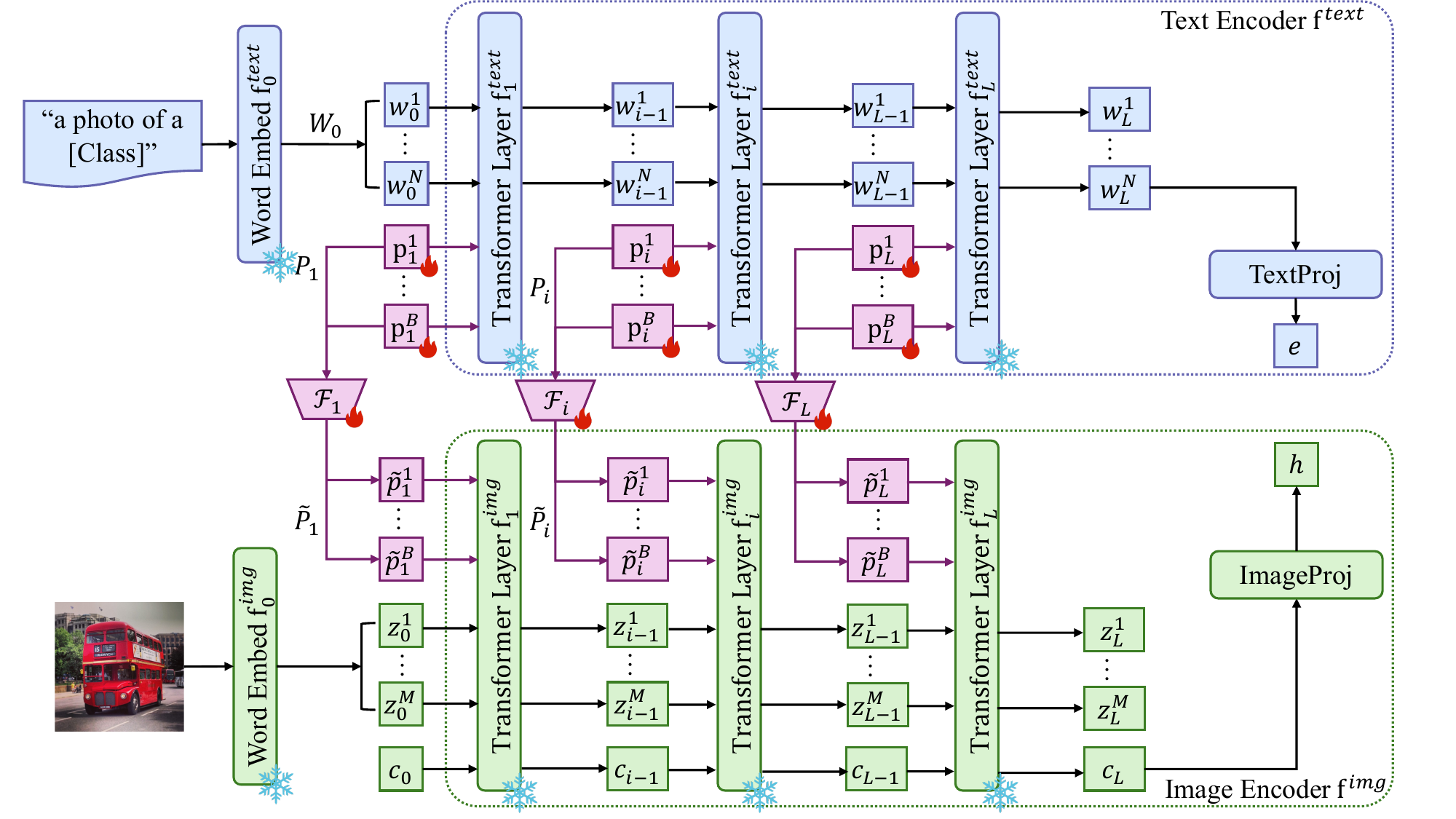}
}
\Description{
Illustration of our multimodal prompt learning. 
}
\caption{
Illustration of our multimodal prompt learning. 
For simplicity, we assume both the text and image encoders have $L$ transformer layers. 
We introduce learnable embeddings at each layer and apply a linear mapping function to couple textual and visual embeddings.
}
\label{fig:maple}
\end{figure*}

% for simplicity，在正文中我们假设text encoder和image encoder拥有一样深度的transformer layer L，which is shown in \cref{fig:maple}。这对于部分CLIP model是适用的，比如VIT/B16.但也有不适用的比如VIT/L14.所以对于这种情况，我们将L设置为一个超参数，对于前L层，如前所述，对于大于L的层，不再mapping embedding，而是the subsequent layers process previous layer learnable embeddings。对于text encoder的层数L_t和image encoder的层数L_v, L_t > L, L_v > L,对于大于L的layer，区别于\cref{eq_maple_t} and \cref{eq_maple_v}, we have 

Our multimodal prompt learning (depicted in \cref{fig:maple}) aims to adapt both the text and image branches of CLIP while keeping the original encoders frozen. 
As outlined in \cref{Sec::MPL}, each encoder comprises $L$ transformer layers, with the text encoder denoted by $\{\mathbf{f}_i^{\text{text}}\}_{i=1}^L$ and the image encoder by $\{\mathbf{f}_i^{\text{img}}\}_{i=1}^L$. 
To achieve efficient optimization and reduce the total number of trainable parameters, we introduce lightweight learnable embeddings for text and mapped embeddings for vision at each layer.

\textbf{Unmatched depth.}
In practice, certain CLIP models (e.g., ViT/B-16) have the same depth $L$ for both text and image encoders. 
However, certain CLIP variants (e.g., ViT/L-14) may have different depths for the text and image encoders. Let \(L_t\) be the number of transformer layers in the text encoder, and let \(L_v\) be the number of transformer layers in the image encoder. We define \(L \leq \min(L_t, L_v)\) and apply our prompt learning formulation only up to layer \(L\). For any layer \(j > L\), we do not introduce additional learnable embeddings or perform mapping updates, and instead allow the subsequent layers to process the previous layer’s output embeddings directly.

Concretely, if $L_t > L$, then for the text encoder layers $j=L+1,\dots,L_t$, we update \cref{eq_maple_t} by:
\begin{align*}
[\bm{\theta}_{j+1},\,\mathbf{W}_j] 
= \mathbf{f}_j^{\mathrm{text}}([\bm{\theta}_j,\,\mathbf{W}_{j-1}]),
\end{align*}
that is, beyond the $L$-th layer, we do not introduce new embeddings nor apply additional mapping functions. 
A similar procedure holds for the image encoder, updated from \cref{eq_maple_v}, when $L_v > L$:
\begin{align*}
[\mathbf{c}_j,\,\mathbf{E}_j,\,\tilde{\bm{\theta}}_{j+1}]
= \mathbf{f}_j^{\mathrm{img}}([\mathbf{c}_{j-1},\,\mathbf{E}_{j-1},\,\tilde{\bm{\theta}}_{j}]).
\end{align*}

This design offers flexibility and ensures that, for models whose text and image encoders have unequal depths, prompt learning remains consistent up to the first $L$ layers. 
Beyond layer $L$, the subsequent layers simply propagate and refine existing prompts without further mapping. 
The key advantage of this architecture is that it {unifies} textual and visual features early in the network while keeping the higher-level representations relatively intact, thus leveraging CLIP’s existing pre-trained knowledge. 
By combining the proposed discriminative representation learning with multimodal prompt learning, we effectively align text and image features, promoting robustness and generalization to unseen generative models.

\section{AI-Generated Image Detection vs OOD Detection}
AI-Generated Image Detection is a specialized application focused specifically on identifying images created by generative models, aiming to expose deepfakes or synthetic media by detecting subtle artifacts or statistical fingerprints left during generation. In contrast, Out-of-Distribution (OOD) Detection~\cite{a,b,c,d,e,f,g,h,j,k,i,l,m} is a broader, fundamental machine learning capability designed to identify any input data that significantly deviates from the model's original training data distribution—whether it's an unknown object class, corrupted data, adversarial examples, or indeed AI-generated images (if the model wasn't trained on them). While AI-generated images often constitute OOD data for models trained solely on real images—making OOD detection techniques applicable—the former focuses narrowly on forensic authenticity verification, whereas the latter addresses general model robustness and safety when encountering novel or unexpected inputs in real-world deployment. Thus, AI-generated image detection can be viewed as a specialized branch of OOD detection, leveraging domain-specific knowledge of generative artifacts.

\section{Datasets and Experimental Settings} \label{sec:appendix_data_exp}

\subsection{Datasets}  \label{sec:appendix_datasets}
\textbf{UniversalFakeDetect dataset.} 
The UniversalFakeDetect dataset is largely composed of images produced by GANs and builds upon the ForenSynths dataset~\cite{DBLP:conf/cvpr/WangW0OE20}. 
Specifically, ForenSynths includes 720K samples: 360K real images and 360K generated ones, with ProGAN as the generator for training data. 
UniversalFakeDetect retains these training conditions but extends the test set to feature multiple generators drawn from ForenSynths: 
ProGAN~\cite{DBLP:conf/iclr/KarrasALL18}, CycleGAN~\cite{DBLP:conf/iccv/ZhuPIE17}, BigGAN~\cite{DBLP:conf/iclr/BrockDS19}, StyleGAN~\cite{DBLP:conf/cvpr/KarrasLA19}, GauGAN~\cite{DBLP:conf/cvpr/Park0WZ19}, StarGAN~\cite{DBLP:conf/cvpr/ChoiCKH0C18}, Deepfakes~\cite{DBLP:conf/iccv/RosslerCVRTN19}, SITD~\cite{DBLP:conf/icml/ChenSWJ18}, SAN~\cite{DBLP:conf/cvpr/DaiCZXZ19}, CRN~\cite{DBLP:conf/iccv/ChenK17}, and IMLE~\cite{DBLP:conf/iccv/LiZM19}. 
Additionally, the dataset incorporates images generated by three diffusion models (Guided Diffusion~\cite{DBLP:conf/nips/DhariwalN21}, GLIDE~\cite{DBLP:conf/icml/NicholDRSMMSC22}, LDM~\cite{DBLP:conf/cvpr/RombachBLEO22}) and one autoregressive model (DALL-E 2~\cite{DBLP:conf/icml/RameshPGGVRCS21}), further expanding upon ForenSynths’ foundation.

\textbf{GenImage dataset.}
GenImage primarily employs diffusion models to generate synthetic images. It draws on real data and labels from ImageNet and relies on Stable Diffusion V1.4 to create its training samples, consisting of fake images and their real counterparts. 
At test time, a diverse set of image generators is included: Stable Diffusion V1.4~\cite{DBLP:conf/cvpr/RombachBLEO22}, Stable Diffusion V1.5~\cite{DBLP:conf/cvpr/RombachBLEO22}, GLIDE~\cite{DBLP:conf/icml/NicholDRSMMSC22}, VQDM~\cite{DBLP:conf/cvpr/GuCBWZCYG22}, Wukong~\cite{DBLP:conf/nips/GuMLHMLYHZJXX22}, BigGAN~\cite{DBLP:conf/iclr/BrockDS19}, ADM~\cite{DBLP:conf/nips/DhariwalN21}, and Midjourney~\cite{midjourney}. 
Altogether, GenImage consists of 1,331,167 real images and 1,350,000 synthetic images. In line with \cite{DBLP:conf/nips/ZhuCYHLLT0H023}, we train on all images produced by Stable Diffusion V1.4 (and the corresponding real images) and then evaluate against all other listed generators. We also include degraded classification experiments on this dataset.

\subsection{Additional Implementation Details}  \label{sec:appendix_imple_details}
We implement \methodName\ using a pre-trained ViT-L/14 CLIP model. 
For multimodal prompt learning, we set the number of learnable embeddings $B=2$ and apply mapping functions up to $L=9$ transformer layers in both the text and image encoders. 
All experiments are trained for 10 epochs with a batch size of 128 and an initial learning rate of 0.002, using SGD with a cosine annealing decay schedule~\cite{DBLP:conf/iclr/LoshchilovH17}, and run on a single NVIDIA L40 GPU.

We utilize the entire training set in GenImage, comprising 162k real and 162k fake images, and fix $\alpha=0.1$ while setting the memory bank size $M=64$.
For the UniversalFakeDetect dataset, following \citet{DBLP:conf/mir/KhanD24}, we reduce the training set to 100k real and 100k fake images (out of the original 360k each) due to the lesser impact of large data size on performance. 
In this setting, we choose $\alpha=0.6$ and setting $M=64$. 

\begin{figure}[t]
    \centering
  \includegraphics[page=1, width=\linewidth]{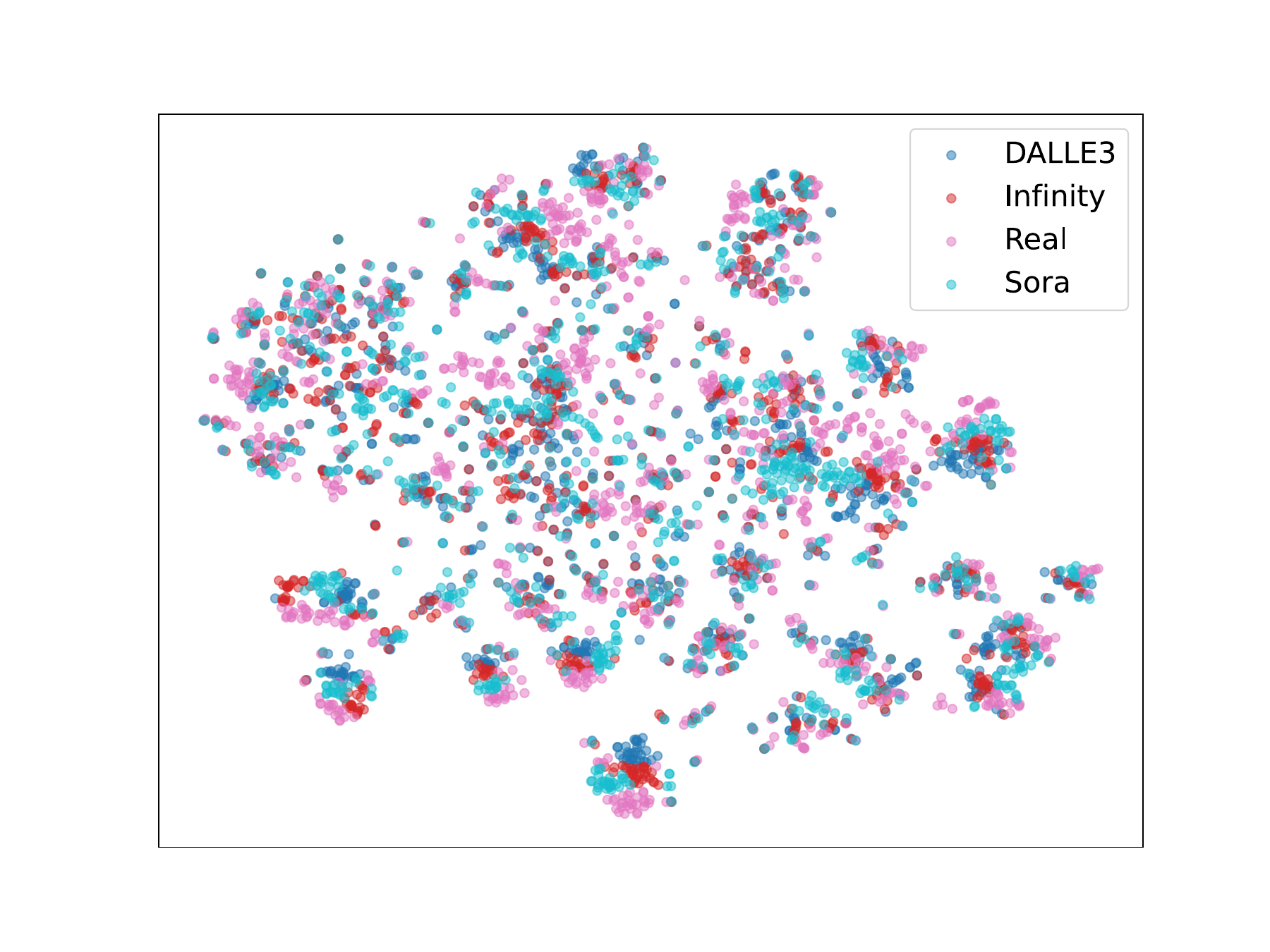}
  \Description{Distribution Analysis via t-SNE.}
  \caption{Distribution Analysis via t-SNE.}
  \label{fig:tsne_visualization_sora}
\end{figure}

\begin{table}[t]
  \centering
  \caption{Zero-shot performance on FLUX.1-dev and SD 3.5.}
  \label{tab:sd35_flux}
  \begin{tabular}{lcccccc}
    \toprule
    & \multicolumn{2}{c}{\textbf{FLUX.1-dev}} 
    & \multicolumn{2}{c}{\textbf{SD 3.5}} 
    & \multicolumn{2}{c}{\textbf{Average}} \\
    \cmidrule(lr){2-3}\cmidrule(lr){4-5}\cmidrule(lr){6-7}
    \textbf{Method}      & Acc   & mAP   & Acc   & mAP   & Acc   & mAP \\
    \midrule
    UnivFD               & 50.4  & 58.0  & 56.7  & 83.7  & 53.5  & 70.8 \\
    CLIPping             & 76.0  & 90.9  & 88.3  & 96.0  & 82.2  & 93.5 \\
    \textbf{MiraGe (Ours)} & \textbf{93.9} & \textbf{99.1} 
                         & \textbf{93.5} & \textbf{98.4} 
                         & \textbf{93.7} & \textbf{98.7} \\
    \bottomrule
  \end{tabular}
\end{table}

\subsection{Additional Details of Collected Datasets Sora, DALLE-3, and Infinity}
\begin{figure*}[t!]
    \centering
    % Top row of images
    \begin{minipage}{0.24\textwidth}
        \centering
        {\normalsize Real} % Add text directly under the image
        \includegraphics[width=\textwidth]{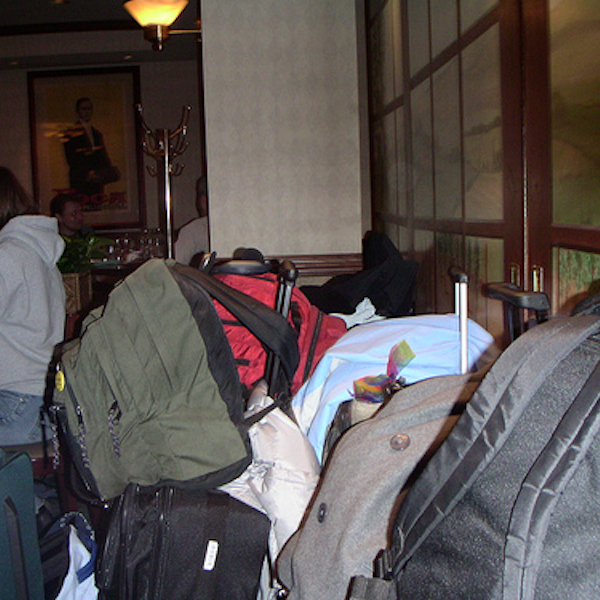}
    \end{minipage}
    \begin{minipage}{0.24\textwidth}
        \centering
        {\normalsize Sora} % Add text directly under the image
        \includegraphics[width=\textwidth]{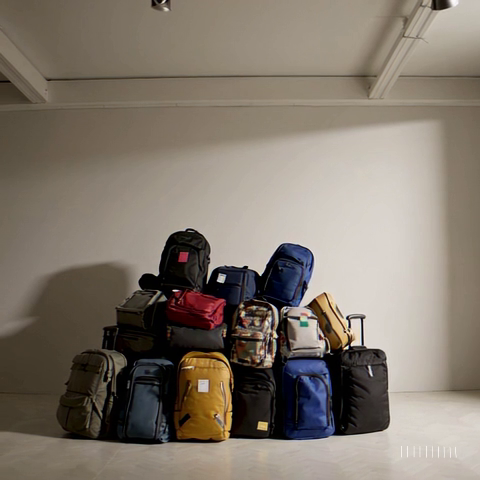}
    \end{minipage}
    \begin{minipage}{0.24\textwidth}
        \centering
        {\normalsize DALLE-3} % Add text directly under the image
        \includegraphics[width=\textwidth]{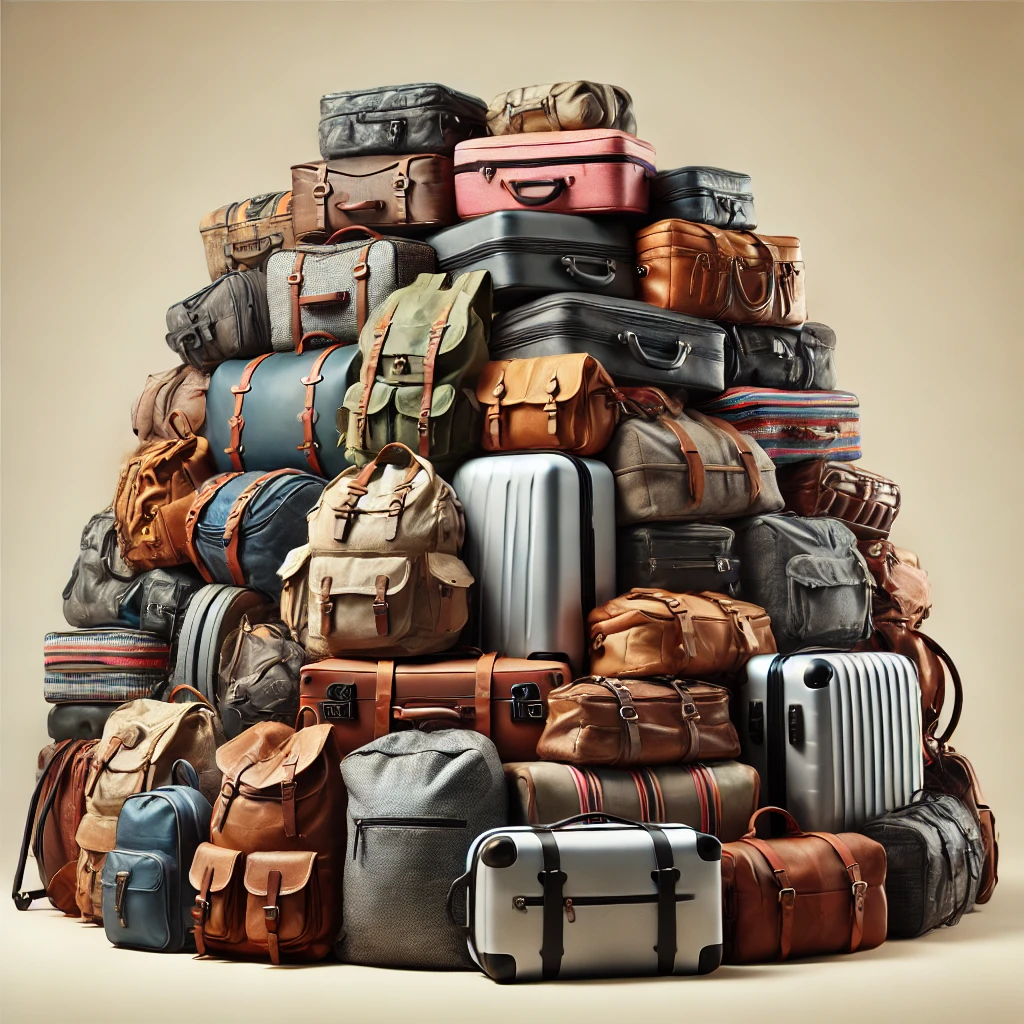}
    \end{minipage}
    \begin{minipage}{0.24\textwidth}
        \centering
        {\normalsize Infinity} % Add text directly under the image
        \includegraphics[width=\textwidth]{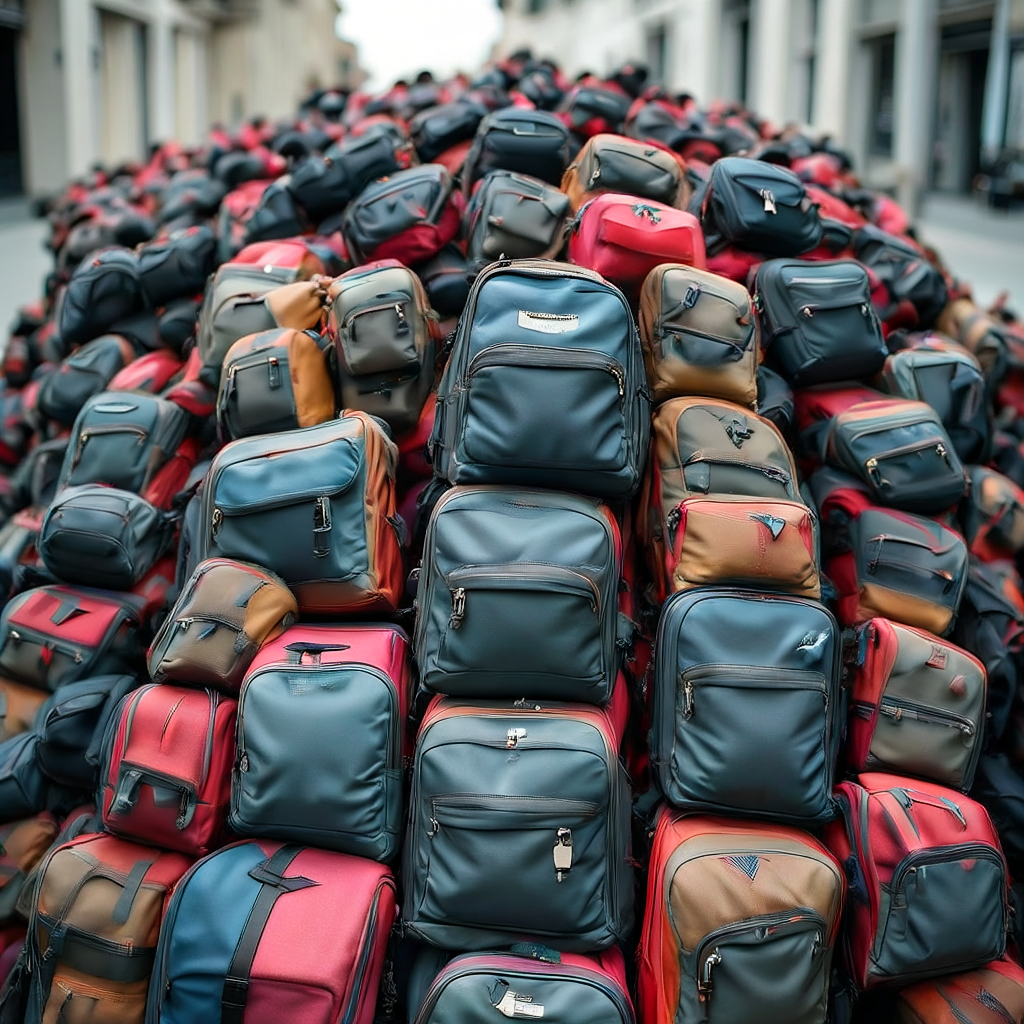}
    \end{minipage}
    % Text for the top row
    Caption: Many backpacks and luggage bags piled up on top of each other.

    % Bottom row of images
    \begin{minipage}{0.24\textwidth}
        \centering
        \includegraphics[width=\textwidth]{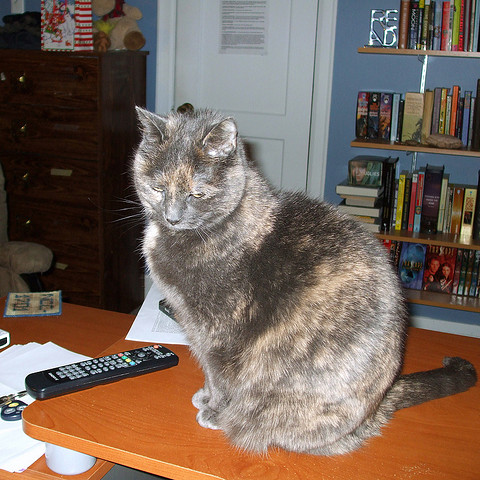}
        % {\footnotesize Real} % Add text directly under the image
    \end{minipage}
    \begin{minipage}{0.24\textwidth}
        \centering
        \includegraphics[width=\textwidth]{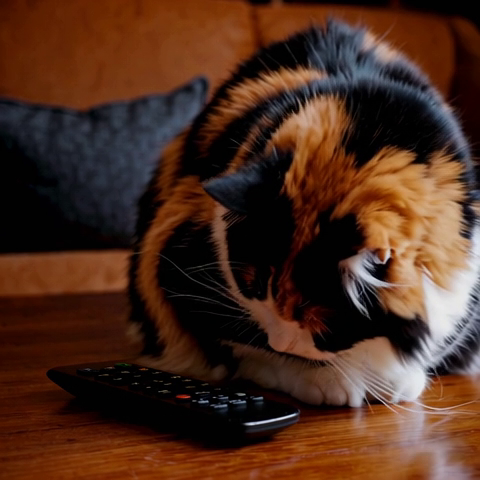}
        % {\footnotesize Sora} % Add text directly under the image
    \end{minipage}
    \begin{minipage}{0.24\textwidth}
        \centering
        \includegraphics[width=\textwidth]{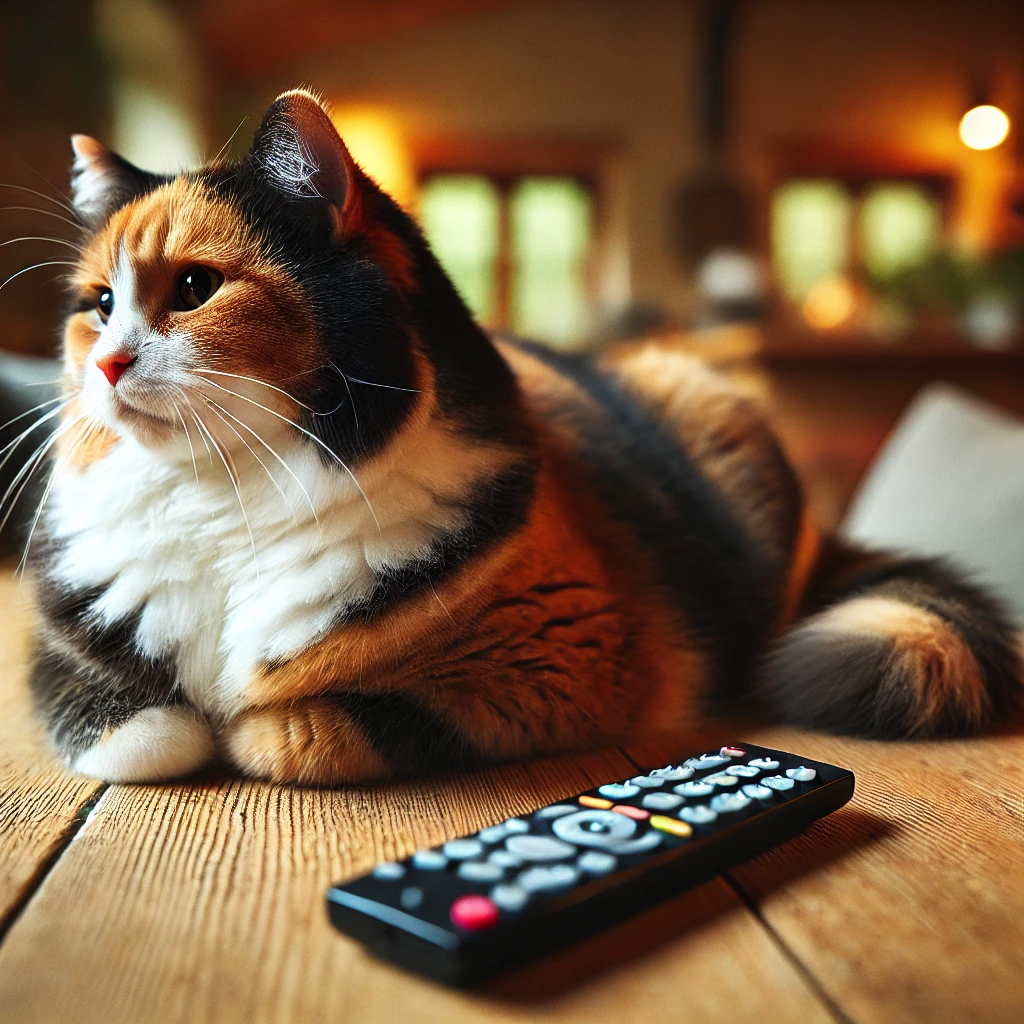}
        % {\footnotesize DALLE-3} % Add text directly under the image
    \end{minipage}
    \begin{minipage}{0.24\textwidth}
        \centering
        \includegraphics[width=\textwidth]{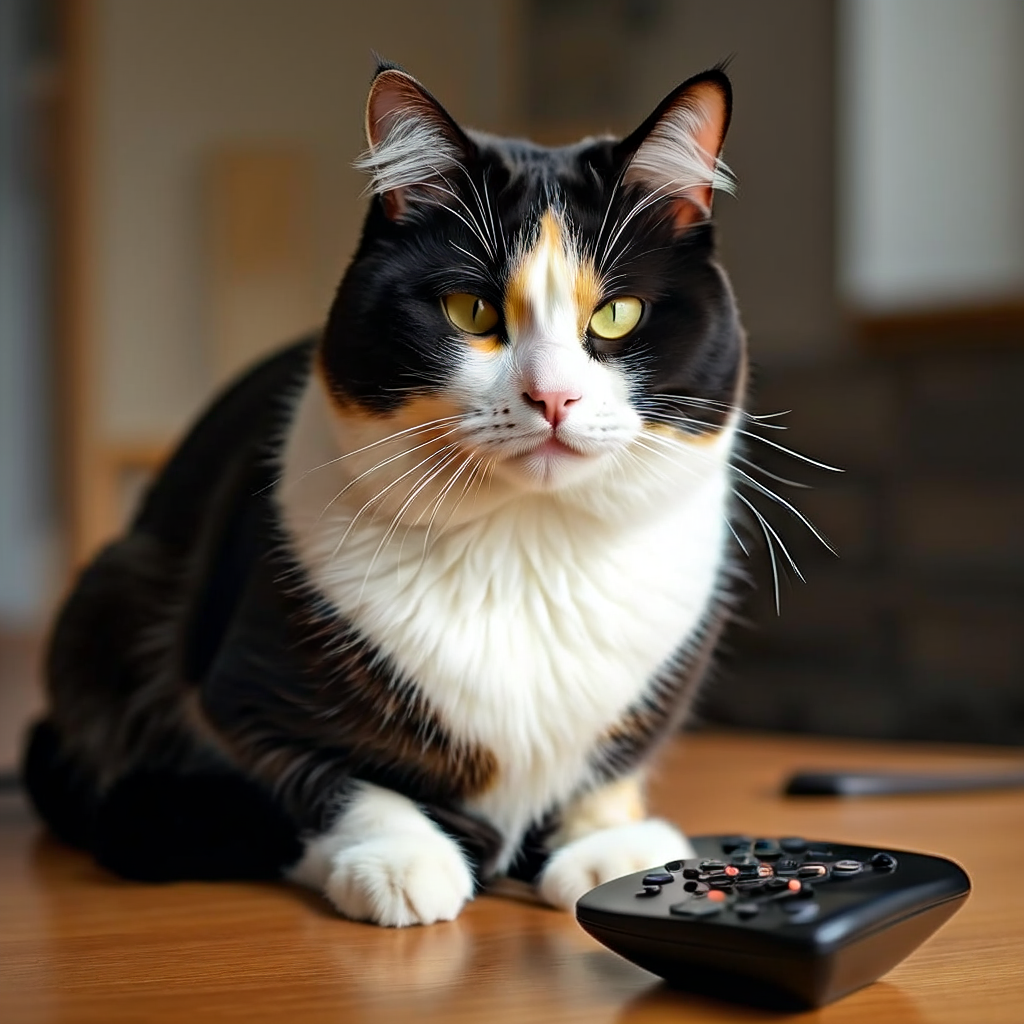}
        % {\footnotesize Infinity} % Add text directly under the image
    \end{minipage}
    % Text for the bottom row
    Caption: A large calico cat sitting on a wooden table next to a remote control.

    \begin{minipage}{0.24\textwidth}
        \centering
        \includegraphics[width=\textwidth]{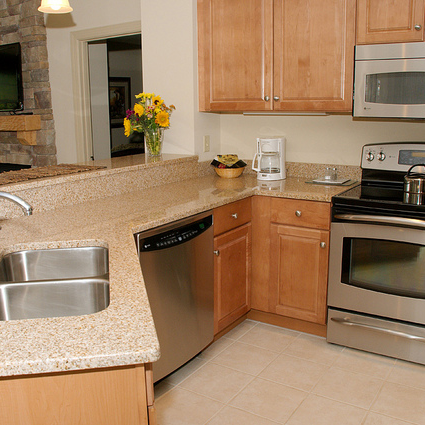}
        % {\footnotesize Real} % Add text directly under the image
        \end{minipage}
    \begin{minipage}{0.24\textwidth}
        \centering
        \includegraphics[width=\textwidth]{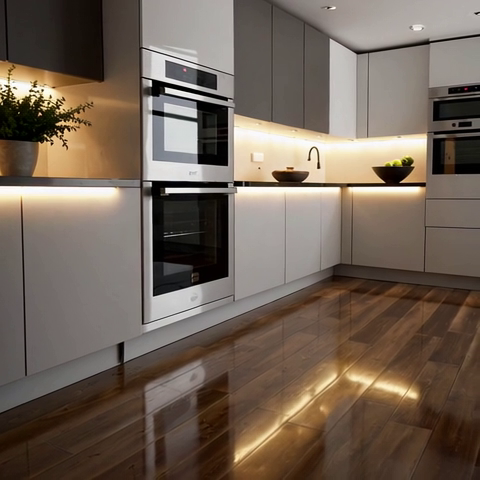}
        % {\footnotesize Sora} % Add text directly under the image
    \end{minipage}
    \begin{minipage}{0.24\textwidth}
        \centering
        \includegraphics[width=\textwidth]{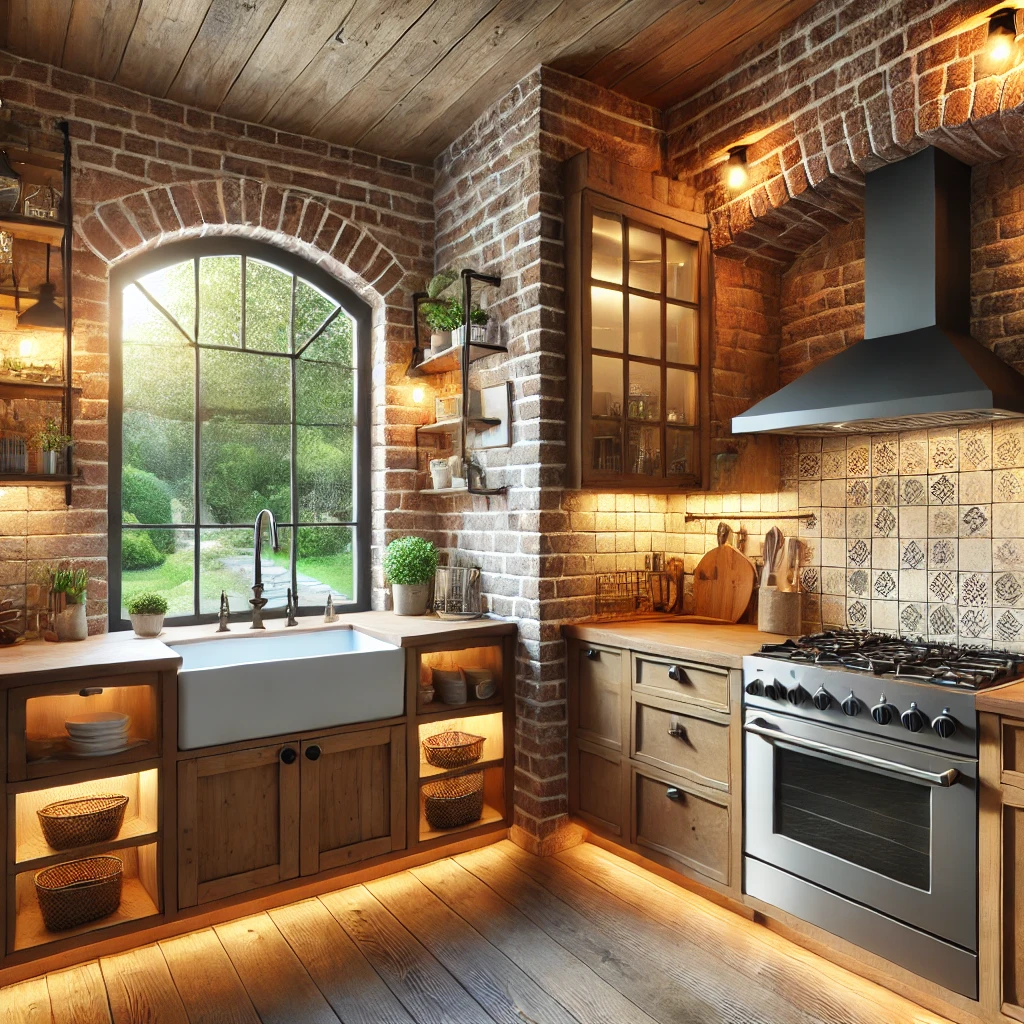}
        % {\footnotesize DALLE-3} % Add text directly under the image
    \end{minipage}
    \begin{minipage}{0.24\textwidth}
        \centering
        % {\footnotesize Infinity} % Add text directly under the image
        \includegraphics[width=\textwidth]{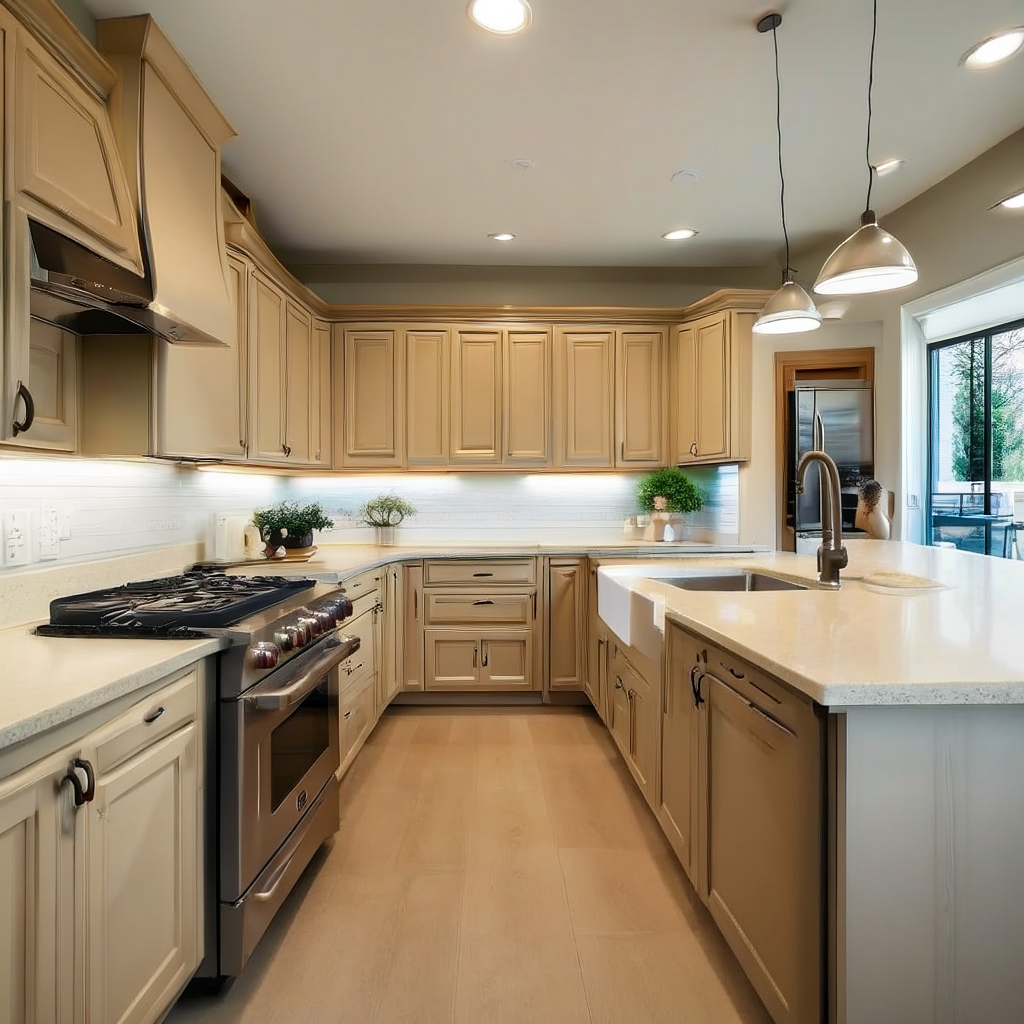}
    \end{minipage}
    % Text for the bottom row
    Caption: A kitchen has light wooden colored cabinetry, beige countertops, and stainless steel appliances in an open house.
    
    % Main caption
    \Description{
    Examples of generated images alongside real images from the MSCOCO dataset.
    }
    \caption{
    Examples of generated images alongside real images from the MSCOCO dataset. The text prompts used for generating the images are displayed below each example. The datasets were constructed using Sora, DALLE-3, and Infinity.
    }
    \label{fig:generated_images}
\end{figure*}
We include three types of AI-generated images from the commercial models Sora \cite{brooks2024video} and DALL-E 3 \cite{betker2023improving}, as well as the emerging AutoRegressive model Infinity \cite{han2024infinityscalingbitwiseautoregressive}. These models are chosen for their distinct generative architectures and diverse output styles, thereby providing a challenging testbed for evaluating the robustness and generalizability of our detector. 

Sora is a commercial text-to-image model specialized in generating high-quality illustrations and artistic renditions. Compared to standard diffusion-based models, Sora often produces more stylized or painterly outputs, which can pose unique challenges for detectors relying on traditional pixel-level artifacts.
DALL-E 3 builds upon OpenAI’s family of generative models, with improvements in resolution and semantic coherence. It leverages a transformer architecture to translate textual prompts into a wide range of visual concepts. Its strong textual-semantic alignment can make detection more difficult, since less-obvious artifacts may be present.
Infinity is a state-of-the-art AutoRegressive model aimed at generating complex scenes. Unlike diffusion-based approaches, it accumulates content incrementally, potentially creating subtle artifacts different from those seen in diffusion outputs.

Following \citet{DBLP:conf/icml/ChenZYY24}, we use the MSCOCO dataset \cite{DBLP:conf/eccv/LinMBHPRDZ14} as our source of real images. We select images and annotations from the 2017 validation set of MSCOCO. We randomly sample 1000 real images from the validation set to represent the real class. Each image in MSCOCO is paired with multiple captions, and we retain the longest caption to preserve maximum textual context for AI image generation. Using these retained captions, we generate 1000 images from each of the three generators (Sora, DALL-E 3, and Infinity), resulting in a total of 3000 AI-generated images.

This procedure ensures a fair comparison between real and generated images while providing a comprehensive benchmark for evaluating generative models. The diversity in approaches, including commercial text-to-image models and an AutoRegressive architecture, enables testing across a wide range of synthetic artifacts, from stylized illustrations to photo-realistic scenes. Retaining the longest captions enriches semantic context, resulting in coherent and realistic outputs, thus increasing detection challenges. This strategy ensures both real and synthetic images span diverse semantic and visual content, thoroughly assessing detector performance.
Some examples of generated images are in \cref{fig:generated_images}.

In addition, we conduct a feature distribution analysis to check the diversity of our sample selection. 
Specifically, we use zero-shot CLIP to embed all 4,000 images (1,000 each from Sora, DALL-E 3, Infinity, and MSCOCO) and then apply t-SNE to project them into a 2D space. 
As illustrated in \cref{fig:tsne_visualization_sora}, images from different generators are broadly interspersed across the semantic space, rather than clustering into narrow or trivial subsets. 
This outcome indicates that our samples capture a wide spectrum of real-world scenarios and generative styles, rather than merely reflecting a convenient subset.
Consequently, our cross-dataset evaluation more faithfully tests the generalization capabilities of each model.

\subsection{Additional Details of Collected Datasets FLUX and SDv3.5}
In addition to classic benchmarks (UniversalFakeDetect, GenImage) and the new Chameleon dataset, and beyond the cross-dataset tests on Sora, DALL-E 3 and Infinity, we have evaluated MiraGe zero-shot on two state-of-the-art generative models, FLUX.1-dev \cite{DBLP:journals/corr/abs-2506-15742} and Stable Diffusion 3.5 \cite{DBLP:conf/icml/EsserKBEMSLLSBP24}), using the same MS-COCO caption protocol (1000 images per model). All detectors remain trained on SD v1.4 without any further fine-tuning.

The result is shown in \cref{tab:sd35_flux}.
Even without any additional training, MiraGe maintains over 93\% accuracy and nearly 99\% mAP on both newly released models, substantially outperforming prior methods. This confirms that our multimodal discriminative representation learning and prompt-based alignment generalize not only across different architectures (diffusion, autoregressive, commercial APIs) but also across successive versions within the same family of generative models.  

\subsection{Correlation between CLIP Score and Detection Accuracy}

We performed a preliminary study to examine how the perceptual quality of generated images affects MiraGe’s detection performance. Specifically, for three recent text-to-image models (Infinity \cite{han2024infinityscalingbitwiseautoregressive}, FLUX.1-dev \cite{DBLP:journals/corr/abs-2506-15742}, and Stable Diffusion 3.5 \cite{DBLP:conf/icml/EsserKBEMSLLSBP24}), we generated 1,000 images each, computed their zero-shot CLIP Scores against the original prompts, and then measured MiraGe’s accuracy on all 3,000 samples. The results are shown in Table~\ref{tab:clipscore}.

\begin{table}[t]
  \centering
  \caption{Correlation between CLIP Score and MiraGe detection accuracy.}
  \label{tab:clipscore}
  \begin{tabular}{lccc}
    \toprule
    Model         & Infinity & FLUX.1-dev & SDv3.5 \\
    \midrule
    CLIP Score    & 0.2624   & 0.2647     & 0.2718 \\
    Accuracy (\%) & 97.5     & 93.9       & 93.5   \\
    \bottomrule
  \end{tabular}
\end{table}

We observe a clear inverse relationship: higher CLIP Scores (indicating closer alignment with the text prompt and more realistic appearance) correspond to lower detection accuracy. This supports the intuitive hypothesis that “the more realistic an AI-generated image, the harder it is to detect.” 

\section{Additional Results on Generalizability} \label{sec:appendix_additional_experiments}
We present additional generalizability results, including comprehensive evaluations on the newly emerging and challenging Chaleneon dataset \cite{DBLP:journals/corr/abs-2406-19435}, as well as extensive degradation studies on the GenImage dataset under conditions such as low resolution, JPEG compression, and Gaussian blurring, and a detailed comparative table for the GenImage dataset.

\begin{table*}[t!]
\caption{
Comparisons on the Chameleon dataset.
Accuracy (\%) of various methods in detecting generated images on the Chameleon dataset.
For each training dataset, the first row presents the overall accuracy on the Chameleon test set, while the second row provides a detailed breakdown as ``fake image / real image accuracy.'' The best results are highlighted in \textbf{bold}, and the second-best are \underline{underlined}.
}
\resizebox{\hsize}{!}
{
\normalsize  

\begin{tabular}{c|cccccccc|c}
\toprule
{Training}              & {CNNSpot}  & {Fusing} & {GramNet} & {LNP}  & {UnivFD } & {DIRE} & {NPR} & {AIDE} & \methodName \\ 
Dataset & \cite{DBLP:conf/cvpr/WangW0OE20} & \cite{DBLP:conf/icip/JuJKXNL22} & \cite{DBLP:conf/cvpr/LiuQT20} & \cite{DBLP:conf/eccv/LiuYBXLG22} & \cite{DBLP:conf/cvpr/OjhaLL23} & \cite{DBLP:conf/iccv/WangBZWHCL23} & \cite{DBLP:conf/cvpr/TanLZWGLW24} & \cite{DBLP:journals/corr/abs-2406-19435}  & (Ours) \\
\midrule
\multirow{2}{*}{{ProGAN}}       & 56.94            & 56.98    & \textbf{58.94}      & 57.11      & 57.22       & \underline{58.19}   & 57.29       & 56.45   & 57.73       \\
                                       & 0.08 / 99.67       & 0.01 / 99.79    & 4.76 / 99.66      & 0.09 / 99.97       & 3.18 / 97.83    & 3.25 / 99.48                         & 2.20 / 98.70   & 0.63 / 98.46 & 1.70 / 99.80     \\ 
                                    \cmidrule(l){1-10} 
\multirow{2}{*}{{SD v1.4}}      & 60.11            & 57.07    & 60.95      & 55.63     & {55.62}       & 59.71     & 58.13    & \underline{61.10}    & \textbf{69.06}      \\
                                       & 8.86 / 98.63                      & 0.00 / 99.96    & 17.65 / 93.50     & 0.57 / 97.01      & 74.97 / 41.09      & 11.86 / 95.67                 & 2.43 / 100.00   & 16.82 / 94.38   & 29.73 / 98.67  \\ \cmidrule(l){1-10} 
\multirow{2}{*}{{All GenImage}} & 60.89          & 57.09    & 59.81      & 58.52       & 
{60.42}       & 57.83      & 57.81   & \underline{63.89}  & \textbf{71.75}        \\
                                       & 9.86 / 99.25    & 0.89 / 99.55    & 8.23 / 98.58      & 7.72 / 96.70       & 85.52 / 41.56      & 2.09 / 99.73                      & 1.68 / 100.00  & 22.40 / 95.06  &  35.91 / 98.68  \\ 
\bottomrule
\end{tabular}
}

\label{table:chameleon}
\end{table*} 

\subsection{Comparisons on Chameleon} 
To evaluate the performance of our proposed method, we utilized the Chameleon dataset, a recently introduced benchmark specifically designed to address the limitations of existing AI-generated image detection datasets. Chameleon stands out for its high realism, diversity, and resolution, making it particularly suitable for realistic and challenging evaluations.
Unlike other datasets that often include AI-generated images with evident artifacts or simplistic prompts, the Chameleon dataset features images that have undergone extensive manual adjustments by AI artists and photographers, ensuring they are highly deceptive to human perception. The dataset also includes a wide variety of categories—human, animal, object, and scene—spanning diverse real-world scenarios. Additionally, all images in Chameleon are high-resolution (up to 4K), providing a rigorous test of the detector's ability to identify subtle differences between AI-generated and real images.
We chose the Chameleon dataset because of its ability to expose the weaknesses of existing detection models. Its images have passed the human ``Turing Test,'' highlighting their resemblance to real-world photography and posing a significant challenge to state-of-the-art detection methods. By incorporating this dataset, we aim to demonstrate the robustness and generalizability of our method in identifying AI-generated images under realistic and demanding conditions.

The results in \cref{table:chameleon} reveal that accuracy declines for all methods on the Chameleon dataset, underscoring its challenging nature with highly realistic and diverse content that closely mimics real photography. While \methodName\ achieves over 90\% accuracy on GenImage and UniversalFakeDetect, its performance drops to 69.06\% on Chameleon. Nevertheless, \methodName\ consistently outperforms all baselines, achieving the highest accuracy when trained on both the SD v1.4 subset (69.06\%) and the full GenImage dataset (71.75\%). Notably, expanding the training data from SD v1.4 to all GenImage subsets further improves performance, highlighting \methodName's ability to leverage diverse generative data for enhanced generalizability. Although \methodName\ does not achieve state-of-the-art results when trained on ProGAN, it still performs competitively, reflecting its adaptability across varied training scenarios. These findings collectively validate the robustness and strong generalization capability of \methodName, even when tested against a complex, high-realism dataset like Chameleon.

\textbf{Discussion.} Although \methodName\ attains the highest accuracy on Chameleon, its performance still hovers around 69--71\%, underscoring the dataset’s deliberately ``human-deceptive'' nature. The AI-generated images in Chameleon have undergone extensive manual refinement by AI artists and photographers to pass the human Turing test, resulting in very few overt generative artifacts. 
Furthermore, the dataset spans a wide array of real-world scenarios---far broader than conventional benchmarks---while offering resolutions from 720P up to 4K. 
Such high-fidelity content demands exceptionally nuanced, fine-grained analysis to tease apart subtle differences between real and synthetic imagery. 
Consequently, these findings reveal a gap between the theoretical strengths of discriminative representation learning and the real-world challenges of detecting highly realistic, diverse, and high-resolution AI-generated images, highlighting an urgent need for further research in robust AI-generated image detection.

\subsection{Robustness Against Degrated Image} 
\begin{table*}[t!]
    \centering
    \caption{
    Performance evaluation on degraded images. Models are trained and tested on the SD V1.4 subset of the GenImage dataset under various degradation scenarios, including low resolution (LR), JPEG compression, and Gaussian blur. The best results are highlighted in \textbf{bold}, and the second-best are \underline{underlined}.
    }
	\label{tab:perturbed_images}
	% \setlength{\tabcolsep}{1.5mm}
    % \resizebox{\hsize}{!}
    {
		\begin{tabular}{cccccccc}
			\toprule
			\multicolumn{1}{c|}{\multirow{2}{*}{Method}} & \multicolumn{6}{c|}{Testing Subset}       & \multicolumn{1}{c}{\multirow{2}{*}{\begin{tabular}[c]{@{}c@{}}Avg\\ Acc.(\%)\end{tabular}}} \\
			\multicolumn{1}{c|}{}   & LR (112) & LR (64) & JPEG (QF=65) & JPEG (QF=30) & Blur ($\sigma$=3) & \multicolumn{1}{c|}{Blur ($\sigma$=5)} & \multicolumn{1}{c}{}                                                                       \\ \hline
            
			\multicolumn{1}{l|}{Spec \cite{DBLP:conf/wifs/0022KC19} }             & 50.0       & 49.9    & 50.8        & 50.4        & 49.9              & \multicolumn{1}{c|}{49.9}              & 50.1                                                                                       \\
			\multicolumn{1}{l|}{F3Net \cite{DBLP:conf/eccv/QianYSCS20}}          & 50.0       & 50.0      & 89.0          & 74.4        & 57.9              & \multicolumn{1}{c|}{51.7}              & 62.1                                                                                       \\
			\multicolumn{1}{l|}{Swin-T \cite{DBLP:conf/iccv/LiuL00W0LG21} }                  & \underline{97.4}     & 54.6    & 52.5        & 50.9        & 94.5              & \multicolumn{1}{c|}{52.5}              & 67.0                                                                                       \\
			\multicolumn{1}{l|}{DIRE  \cite{DBLP:conf/iccv/WangBZWHCL23} }              & {64.1}     & {53.5}    &85.4        & 65.0        & 88.8              & \multicolumn{1}{c|}{{56.5}}              & {68.9}                                                                                       \\   
			\multicolumn{1}{l|}{DeiT-S \cite{DBLP:conf/icml/TouvronCDMSJ21}}                  & 97.1     & 54.0      & 55.6        & 50.5        & 94.4              & \multicolumn{1}{c|}{67.2}              & 69.8                                                                                       \\
			\multicolumn{1}{l|}{ResNet-50 \cite{DBLP:conf/cvpr/HeZRS16} }                & 96.2     & 57.4    & 51.9        & 51.2        & \textbf{97.9}              & \multicolumn{1}{c|}{69.4}              & 70.6                                                                                       \\
			\multicolumn{1}{l|}{CNNDet  \cite{DBLP:conf/cvpr/WangW0OE20}}       & 50.0       & 50.0      & \textbf{97.3}        & \textbf{97.3}        & \underline{97.4}              & \multicolumn{1}{c|}{{77.9}}              & {78.3}                                                                                       \\
			\multicolumn{1}{l|}{UnivFD  \cite{DBLP:conf/cvpr/OjhaLL23} }              & {88.2}     & {78.5}    &85.8        & 83.0        & 69.7              & \multicolumn{1}{c|}{{65.7}}              & {78.3} \\ 
			\multicolumn{1}{l|}{GramNet \cite{DBLP:conf/cvpr/LiuQT20} }                  & \textbf{98.8}     & \textbf{94.9}    & 68.8        & 53.4        & 95.9              & \multicolumn{1}{c|}{\underline{81.6}}              & \underline{82.2}                      \\                                                    
                \midrule
			\rowcolor[HTML]{EFEFEF} \multicolumn{1}{l|}{\methodName\ (Ours)}    & {92.9}  & \underline{80.4} & \underline{95.7}       & \underline{93.9}       & \underline{97.4}      & \multicolumn{1}{c|}{\textbf{90.9}}             & \textbf{91.9}                                                                                       \\ 			
			
			\bottomrule

		\end{tabular}
	}
\end{table*}
In real-world scenarios, images frequently undergo perturbations such as resolution reduction, JPEG compression, and blurring during transmission and interaction~\cite{DBLP:conf/cvpr/WangW0OE20}. To assess how these degradations affect AI-generated image detection, we downsample images to resolutions of 112 and 64, apply JPEG compression with quality factors (QF) of 65 and 30, and introduce Gaussian blur with \(\sigma\)=3 and \(\sigma\)=5. As shown in \cref{tab:perturbed_images}, these disruptions weaken the discriminative artifacts of generative models, making it more difficult to differentiate real from AI-generated images and substantially reducing the performance of existing detectors.

To enhance robustness to such unseen perturbations, we employ an array of data augmentations during training, including random crops and resizes, Gaussian noise, Gaussian blur, random rotations, JPEG compression with random quality, brightness and contrast adjustments, and random grayscale conversions. Despite the demanding conditions of low-resolution input, strong compression artifacts, and severe blurring, our method maintains the highest average accuracy of 91.9\%. This superior performance underscores the effectiveness of our multimodal design in capturing and utilizing both semantic and noise-related cues, even when pixel distributions are heavily distorted.

\begin{table*}[t]
    \centering
    \caption{
    Comprehensive comparison of accuracy (\%) between our method and other methods. All methods were trained on the GenImage SDv1.4 dataset and evaluated across different testing subsets. The best results are highlighted in \textbf{bold}, and the second-best are \underline{underlined}.
    }
    % \resizebox{.99\linewidth}{!}
    {
    \begin{tabular}{lccccccccc}
        \toprule
        {Method} & {Midjourney} & {SDv1.4} & {SDv1.5} & {ADM} & {GLIDE} & {Wukong} & {VQDM} & {BigGAN} & {Avg (\%)} \\
        \midrule
        CNNDet \cite{DBLP:conf/cvpr/WangW0OE20}          & 52.8 & 96.3 & {99.5} & 50.1 & 39.8 & 78.6 & 53.4 & 46.8 & 64.7  \\
        F3Net \cite{DBLP:conf/eccv/QianYSCS20}           & 50.1 & 99.2 & \textbf{99.9} & 49.9 & 39.0 & \underline{99.1} & 60.9 & 48.9 & 68.7  \\
        Spec \cite{DBLP:conf/wifs/0022KC19}              & 52.0 & 99.4 & 99.2 & 49.7 & 48.9 & 94.8 & 55.6 & 49.6 & 68.8  \\
        GramNet \cite{DBLP:conf/cvpr/LiuQT20}            & 54.2 & 99.2 & 99.1 & 50.3 & 54.6 & 98.0 & 50.8 & 51.7 & 69.9  \\
        DIRE \cite{DBLP:conf/iccv/WangBZWHCL23}          & 50.4 & \textbf{100.0} & \textbf{99.9} & 52.5 & 62.7 & 56.5 & 52.4 & 59.5 & 71.2 \\
        DeiT-S \cite{DBLP:conf/icml/TouvronCDMSJ21}      & 55.6 & \underline{99.9} & \textbf{99.9} & 49.8 & 58.1 & {98.9} & 56.9 & 53.5 & 71.6  \\
        ResNet-50 \cite{DBLP:conf/cvpr/HeZRS16}          & 54.9 & \underline{99.9} & 99.7 & 53.5 & 61.9 & 98.2 & 56.6 & 52.0 & 72.1  \\
        Swin-T \cite{DBLP:conf/iccv/LiuL00W0LG21}        & 62.1 & \underline{99.9} & \textbf{99.9} & 49.8 & 67.6 & \underline{99.1} & 62.3 & 57.6 & 74.8  \\
        UnivFD \cite{DBLP:conf/cvpr/OjhaLL23}            & \underline{91.5} & 96.4 & 96.1 & 58.1 & 73.4 & 94.5 & 67.8 & 57.7 & 79.4  \\
        GenDet \cite{DBLP:journals/corr/abs-2312-08880}  & 89.6 & 96.1 & 96.1 & 58.0 & 78.4 & 92.8 & 66.5 & 75.0 & 81.6  \\
        CLIPpping \cite{DBLP:conf/mir/KhanD24}           & 76.2 & 93.2 & 92.8 & 71.6 & 87.5 & 83.3 & 75.4 & 75.8 & 82.0  \\
        De-fake \cite{DBLP:conf/ccs/ShaLYZ23}            & 79.9 & 98.7 & 98.6 & 71.6 & 70.9 & 78.3 & 74.4 & \underline{84.7} & 84.7  \\
        LaRE \cite{DBLP:conf/cvpr/LuoDYD24}	    & 74.0 	& \textbf{100.0} 	& \textbf{99.9} 	& 61.7 	& 88.5 	& \textbf{100.0} & \textbf{97.2} 	& 68.7  & 86.2 \\
        AIDE \cite{DBLP:journals/corr/abs-2406-19435}    & 79.4 & 99.7 & \underline{99.8} & 78.5 & \underline{91.8} & 98.7 & 80.3 & 66.9 & 86.9  \\
        DRCT \cite{DBLP:conf/icml/ChenZYY24}             & \underline{91.5} & 95.0 & 94.4 & {79.4} & 89.2 & {94.7} & {90.0} & 81.7 & {89.5} \\
        ESSP \cite{DBLP:journals/corr/abs-2402-01123}    & 82.6 & 99.2 & 99.3 & 78.9 & 88.9 & 98.6 & \underline{96.0} & 73.9 & 89.7  \\
        NPR \cite{DBLP:conf/cvpr/TanLZWGLW24}            & \textbf{91.7} & 97.4 & 94.4 & \textbf{87.8} & \textbf{93.2} & 94.0 & 88.7 & 80.7 & \underline{91.0}  \\
        \midrule
        \rowcolor[HTML]{EFEFEF} \methodName\ (Ours)      & 83.2 & 98.8 & 98.5 & \underline{82.7} & 91.3 & {97.6} & {92.4} & \textbf{96.5} & \textbf{92.6} \\
        \bottomrule
    \end{tabular}
    \label{tab:Performance_Comparison_on_GenImage_Comprehensive}
    }
\end{table*}

\subsection{Comprehensive Comparisons on GenImage}
To further underscore the effectiveness of \methodName, we expand our evaluation on the GenImage dataset to include a broader set of baseline methods, covering both classic and recently proposed detectors. All approaches are trained on the SDv1.4 subset and tested on eight distinct generative subsets: Midjourney, SDv1.4, SDv1.5, ADM, GLIDE, Wukong, VQDM, and BigGAN. As shown in \cref{tab:Performance_Comparison_on_GenImage_Comprehensive}, \methodName\ achieves the highest average accuracy of 92.6\%, demonstrating robust generalization across diverse generative styles.

Compared to earlier analyses, this comprehensive comparison incorporates additional methods such as ESSP \cite{DBLP:journals/corr/abs-2402-01123} and NPR \cite{DBLP:conf/cvpr/TanLZWGLW24}, offering deeper insights into the relative strengths and weaknesses of each approach. While some detectors (e.g., UnivFD \cite{DBLP:conf/cvpr/OjhaLL23}, NPR \cite{DBLP:conf/cvpr/TanLZWGLW24}, and DRCT \cite{DBLP:conf/icml/ChenZYY24}) exhibit competitive results on subsets closely resembling their training distributions, their performance degrades when confronted with more structurally distinct generators (e.g., BigGAN). 
In contrast, \methodName\ maintains strong accuracy across all subsets, with a notable 96.5\% on BigGAN. We attribute this resilience to both our multimodal prompt learning and discriminative representation learning, which captures generator-agnostic features by aligning image and text embeddings.

\section{Additional Ablation Study} \label{sec:appendix_additional_ablation}

\subsection{Mapping Functions}
We investigate the impact of different coupling functions $\tilde{\bm{\theta}}_i = \mathcal{F}_i(\bm{\theta}_i;\bm{\theta}^{\mathcal{F}}_i)$ that map text-anchor embeddings into the vision prompt space. Inspired by CLIP’s original design, where a single linear projection suffices for cross-modal alignment, we compare four candidates:
(1) {1-layer linear (ours)}: a single linear layer mapping ${\bm{\theta}}_i \rightarrow \tilde{\bm{\theta}}_i$.
(2) {2-layer MLP}: \texttt{Linear} \(\rightarrow\) \texttt{ReLU} \(\rightarrow\) \texttt{Linear}.
(3) {Cross-modal attention}: a lightweight self-attention block between $\bm{\theta}_i$ and visual prompts.
(4) {Reverse mapping}: a 1-layer linear mapping from vision prompts $\tilde{\bm{\theta}}_i$ back to word-embedding space ${\bm{\theta}}_i$.

\begin{table}[t]
  \caption{Ablation study on mapping functions.}
  \label{tab:appendix_mapping}
  \begin{tabular}{lcc}
    \toprule
    Mapping function & GenImage Avg. Acc. & Chameleon Acc. \\
    \midrule
    1-layer linear (ours)  & 92.6 & 69.1 \\
    2-layer MLP  & 	92.8 & 68.5 \\
    Cross-modal attention  & 92.2 & 68.2 \\
    $\tilde{\bm{\theta}}_i \rightarrow {\bm{\theta}}_i$  & 90.8 & 67.9 \\
  \bottomrule
\end{tabular}
\end{table}

The result is shown in \cref{tab:appendix_mapping}. While higher-capacity mappers (MLP or attention) match or slightly improve GenImage accuracy, they degrade Chameleon performance (e.g., –0.9\% with attention), suggesting overfitting to training-generator artifacts. The simple linear map achieves equally strong or better results with minimal extra parameters, and thus remains our default choice.

\subsection{Computational Efficiency}

While prompt learning dramatically reduces trainable parameters, we acknowledge the importance of quantifying overall computational cost. As described in \cref{sec:appendix_bank}, our memory bank mechanism incurs only about 2 MB of extra GPU memory (float16), rendering its overhead negligible.
To contextualize MiraGe’s runtime and latency, we compared against the CLIPping prompt-learning baseline under identical conditions (200k images from UniversalFakeDetect, 10 epochs on an NVIDIA L40 GPU), the result is shown in \cref{tab:appendix_computation}.
Despite integrating a discriminative loss and memory bank, MiraGe’s total training time remains nearly identical to CLIPping’s prompt-learning setup (203 min vs.\ 200 min), and its inference latency increases by less than 0.2 ms per image. In contrast, full fine-tuning of CLIP requires nearly five times more training time. These results confirm that MiraGe delivers substantial detection improvements with only minimal additional computational cost over lightweight prompt-based methods.

\begin{table}[t]
  \centering
  \caption{Comparison of detection performance, training time (min), and inference latency (ms/image).}
  \label{tab:appendix_computation}
  \resizebox{\hsize}{!}{
  \begin{tabular}{lcccc}
    \toprule
    Method                              & mAP & Acc & Training Time & Latency \\
    \midrule
    CLIPping (prompt learning)          & 95.2     & 87.4     & 200                 & 3.45               \\
    MiraGe (Ours)                       & 97.8     & 92.0     & 203                 & 3.56               \\
    CLIPping (full fine-tuning)         & 93.5     & 86.7     & 980                 & 3.45               \\
    \bottomrule
  \end{tabular}
  }
\end{table}

\subsection{Effect of Prompt Variations}

We evaluate the sensitivity of MiraGe to different text anchor formulations by testing three prompt templates on the UniversalFakeDetect benchmark. Detection accuracy (Acc) and mean Average Precision (mAP) are reported in \cref{tab:appendix_prompt_variations}.

\begin{table}[t]
  \centering
  \caption{Impact of prompt template variations on detection performance.}
  \label{tab:appendix_prompt_variations}
  \begin{tabular}{lccc}
    \toprule
    Prompt template                              & Acc & mAP & $\Delta$ \\
    \midrule
    a photo of a real / fake                     & 92.9     & 98.3     & —        \\
    a photo of an authentic / synthetic          & 92.7     & 98.2     & –0.2 / –0.1 \\
    an original / generated image                & 92.5     & 98.0     & –0.4 / –0.3 \\
    \bottomrule
  \end{tabular}
\end{table}

From the result, we can see that swapping “real/fake” for synonyms or rephrasing the template decreases accuracy by at most 0.4\% and mAP by at most 0.3\%.
The simplest binary labels (Real / Fake) achieve near–optimal performance, indicating that MiraGe does not depend on elaborate prompt engineering.
This is because to reduce reliance on any fixed vocabulary, MiraGe employs deep multimodal prompt learning. The prompt embeddings are learnable and adapt during training, so that the final semantic prototypes become data-driven centers rather than the static CLIP embeddings of Real and Fake. This automatic adaptation underlies the observed robustness to prompt variations.
In summary, these additional experiments confirm that MiraGe maintains stable, high performance across reasonable prompt synonyms and template rewrites.

\subsection{Additional Analysis on Training Data Scale} \label{sec:appendix_data_scale}

In \cref{tab:train_size_effect}, we observed that MiraGe’s accuracy rises sharply as the training set increases from 1k to 20k images, but then levels off or even dips slightly up to 200k images. 
This can be explained by three interrelated factors. 
First, the pre-trained CLIP backbone provides highly sample-efficient multi-modal embeddings that capture the core distinctions between real and synthetic images with only a few tens of thousands of examples; for instance, CoOp demonstrates strong few-shot performance with as few as 16 labeled samples per class. 
Second, once the model has seen a representative variety of generative artifact patterns, additional images tend to repeat previously encountered patterns and contribute little new information—indeed, excessive redundancy can introduce noise or low-quality examples, resulting in diminishing returns and occasional performance drops. 
Third, because all fake training images originate from a single ProGAN generator, enlarging that ProGAN pool has limited impact on the model’s ability to detect outputs from other architectures; after ProGAN artifacts are well covered, further gains in cross-generator generalization depend more on embedding diversity across different model families than on simply adding more ProGAN samples. 

In summary, once MiraGe has encountered a sufficiently diverse real-vs-fake sample set, its generalization to unseen generators is governed primarily by the quality and diversity of the learned embeddings, rather than by further increases in dataset size.  

\end{document}